\documentclass[twoside]{article}

\usepackage[accepted]{aistats2022}
\usepackage{booktabs}
\usepackage{natbib}

\usepackage[utf8]{inputenc}
\usepackage{algpseudocode}
\usepackage[T1]{fontenc}
\usepackage{subcaption}
\usepackage{algorithm}
\usepackage{mathtools}
\usepackage{hyperref}
\usepackage{graphicx}
\usepackage{csquotes}
\usepackage{makecell}
\usepackage{multirow}
\usepackage{multicol}
\usepackage{caption}
\usepackage{amsmath}
\usepackage{amssymb}
\usepackage{amsthm}
\usepackage{float}
\usepackage{ulem}
\usepackage{bbm}

\usepackage{xr}

\newtheorem{assumption}{Assumption}

\newtheorem{theorem}{Theorem}
\newtheorem{lemma}[theorem]{Lemma}
\newtheorem{corollary}[theorem]{Corollary}

\newcommand{\R}{\mathbb{R}}

\newcommand{\B}{\mathbb{B}}
\newcommand{\E}{\mathbb{E}}
\newcommand{\F}{\mathcal{F}}
\newcommand{\I}{\mathbb{I}}

\renewcommand{\S}{\mathbb{S}}
\newcommand{\inner}[2]{\langle #1, #2 \rangle}

\DeclareMathOperator{\sign}{sign}
\DeclareMathOperator{\prox}{prox}

\DeclareMathOperator*{\argmin}{arg\,min}
\DeclareMathOperator{\unif}{unif}

\begin{document}

%
\runningtitle{Semi-Implicit Hybrid Gradient Methods with Application to Adversarial Robustness}

%
\runningauthor{Beomsu Kim, Junghoon Seo}

\twocolumn[

\aistatstitle{Semi-Implicit Hybrid Gradient Methods\\with Application to Adversarial Robustness}

\aistatsauthor{Beomsu Kim \And Junghoon Seo}

\aistatsaddress{ Dept. of Mathematical Sciences, KAIST \And  SI Analytics, Inc.} ]

\begin{abstract}
Adversarial examples, crafted by adding imperceptible perturbations to natural inputs, can easily fool deep neural networks (DNNs). One of the most successful methods for training adversarially robust DNNs is solving a nonconvex-nonconcave minimax problem with an adversarial training (AT) algorithm. However, among the many AT algorithms, only Dynamic AT (DAT) and You Only Propagate Once (YOPO) guarantee convergence to a stationary point. In this work, we generalize the stochastic primal-dual hybrid gradient algorithm to develop semi-implicit hybrid gradient methods (SI-HGs) for finding stationary points of nonconvex-nonconcave minimax problems. SI-HGs have the convergence rate $O(1/K)$, which improves upon the rate $O(1/K^{1/2})$ of DAT and YOPO. We devise a practical variant of SI-HGs, and show that it outperforms other AT algorithms in terms of convergence speed and robustness.
\end{abstract}

\section{INTRODUCTION} \label{sec:intro}

Adversarial examples, crafted by adding imperceptible adversarial perturbations to natural inputs, can easily fool deep neural networks (DNNs) \citep{goodfellow2015}. One of the most successful methods for learning adversarially robust DNNs is adversarial training (AT) \citep{athalye2018}. Given a dataset $\{(x_i,y_i)\}_{i = 1}^n$ comprised of $n$ input-label pairs or batches, a vector of perturbations $\delta = (\delta_1, \ldots, \delta_n)$, perturbation radius $\epsilon > 0$, DNN parameters $w$, and a loss function $\ell$, AT solves the nonconvex-nonconcave minimax problem
\begin{align} \label{eq:AT_intro}
\min_w \max_{\|\delta\|_\infty \leq \epsilon} \frac{1}{n} \sum_{i = 1}^n \ell(x_i + \delta_i, y_i, w).
\end{align}

However, among the large number of AT algorithms \citep{silva2020}, many do not have theoretical convergence guarantees. Some such algorithms even exhibit a failure mode called catastrophic overfitting, where the DNN accuracy on adversarial examples generated by multiple steps of projected gradient ascent (PGD) drops to a low value in the middle of training \citep{co2020}.

To the best of our knowledge, there are only two AT algorithms guaranteed to converge to a stationary point of (\ref{eq:AT_intro}): Dynamic AT (DAT) \citep{wang2019} and You Only Propagate Once (YOPO) \citep{zhang2019,seidman2020}. Under certain assumptions on the loss function and stochastic gradients, DAT and YOPO decrease the squared saddle subdifferential norm with rate $O(1/K^{1/2})$ up to an additive constant.

\begin{table*}[t]
\caption{Convergence rates w.r.t. squared saddle subdifferential norm for AT algorithms. \enquote{---} means N/A. \enquote{DT} means deterministic, \enquote{S} means stochastic, and \enquote{DS} means doubly-stochastic. $K$ is the number of iterations. We omit additive constants for the convergence rates. For SSI-HG, smoothness of $\ell(\cdot,y_i,\cdot)$ is actually more than what we need to prove the $O(1/K)$ rate (c.f. Assumption \ref{assump:1} (b) and Theorem \ref{thm:SSI-HG1}). Yet, we have written the stronger assumption in this table since we use the smoothness of $\ell(\cdot,y_i,\cdot)$ in the process of interpreting SSI-HG as MGDA (see Section \ref{sec:SI-HG}).}
\label{table:AT_comp}
\centering
\vspace{1.0em}
\resizebox{\textwidth}{!}{
\begin{tabular}{c c c c c}
\toprule
\textbf{AT Algorithm} & \textbf{Algorithm Type} & \textbf{Assumption} & \textbf{Setting} & \textbf{Convergence} \\
\cmidrule{1-5}
\makecell{FGSM AT \\ \citep{goodfellow2015}} & GDA & --- & DS & --- \\
\cmidrule{1-5}
\makecell{PGD AT \\ \citep{madry2018}} & MGDA & $\ell$ continuously diff. in $w$ & DT & Global Convergence \\
\cmidrule{1-5}
\makecell{DAT \\ \citep{wang2019}} & MGDA & \makecell{$\ell(\cdot,y_i,\cdot)$ smooth \\ $\ell(\cdot,y_i,w)$ locally strongly concave \\ Stochastic gradient bounded variance} & DS & $O(1/K^{1/2})$ \\
\cmidrule{1-5}
\makecell{YOPO \\ \citep{seidman2020}} & MGDA & \makecell{$\ell(\cdot,y_i,\cdot)$ smooth \\ $\ell(\cdot,y_i,w)$ locally strongly concave \\ Stochastic gradient bounded variance} & DS & $O(1/K^{1/2})$ \\
\cmidrule{1-5}
\makecell{SSI-HG \\ (Ours)} & MGDA & \makecell{$\ell(\cdot,y_i,\cdot)$ smooth \\ Weak MVI has solution} & S & $O(1/K)$ \\
\bottomrule
\end{tabular}}
\end{table*}

In this work, we take a step towards AT algorithms with better convergence guarantees. To this end, we consider minimax optimization problems of the form
\begin{align} \label{eq:minimax}
\min_w \max_\delta f(w) + \phi(w,\delta) - g(\delta)
\end{align}
where
\begin{align}
\phi(w,\delta) = \sum_{i = 1}^n \phi_i(w,\delta_i), \qquad g(\delta) = \sum_{i = 1}^n g_i(\delta_i)
\end{align}
and $w \in \R^m$, $\delta = (\delta_1, \ldots, \delta_n) \in \R^{d_1} \times \cdots \times \R^{d_n}$, $f, g_i$ are convex, and $\phi$ is nonconvex-nonconcave. The AT problem (\ref{eq:AT_intro}) is a special case of this template (c.f. Equation (\ref{eq:AT})). We propose two semi-implicit hybrid gradient methods (SI-HGs) which solve (\ref{eq:minimax}) by alternating between a hybrid gradient descent step on $w$ and an implicit ascent step on $\delta$. The first SI-HG is stochastic\footnote{We will refer to methods which use full gradients to update both $w$ and $\delta$ as being \textit{deterministic}, methods which use full gradients to update $w$ and stochastic gradients to update $\delta$ as being \textit{stochastic}, and methods which use stochastic gradients to update both $w$ and $\delta$ as being \textit{doubly-stochastic}.} SI-HG (SSI-HG), which generalizes stochastic primal-dual hybrid gradient (SPDHG) \citep{chambolle2018}. The second SI-HG is deterministic SI-HG (DSI-HG), which is the $n = 1$ version of SSI-HG.

 SI-HGs decrease the squared saddle subdifferential norm with rate $O(1/K)$ under the weak Minty variational inequality (MVI) condition \citep{wmvi2021}. This improves upon the aforementioned $O(1/K^{1/2})$ convergence rate of DAT and YOPO. We also prove that SI-HGs achieve a linear rate under the strong MVI condition \citep{zhou2017,song2020}. Our work is a solid step towards AT algorithms with better convergence properties. A variety of experiments substantiate this claim.

We have the following key contributions.
\begin{itemize}
\item \textbf{SSI-HG.} We propose SSI-HG, which generalizes SPDHG to the nonconvex-nonconcave minimax optimization setting. We show that SSI-HG can be interpreted as a stochastic multi-step gradient descent ascent method (MGDA). Under the weak MVI assumption, we prove that SSI-HG decreases the squared saddle subdifferential norm with rate $O(1/K)$. We also prove linear convergence under the strong MVI assumption. Our work improves upon the convergence rates for other AT algorithms (Table \ref{table:AT_comp}).
\item \textbf{DSI-HG.} When $n = 1$, SSI-HG becomes DSI-HG. DSI-HG inherits all the convergence results for SSI-HG. Hence, if we interpret DSI-HG as MGDA, our work improves upon previous convergence results for MGDA in the nonconvex-nonconcave setting (Table \ref{table:MGDA_comp}).
\item \textbf{Application of SI-HGs to AT.} We extend the theoretical development behind SI-HGs to the AT setting. Specifically, we develop a minibatch version of SI-HG (MSI-HG). We also propose a heuristic for solving the implicit step in the AT setting. In experiments, we demonstrate MSI-HG indeed converges faster and achieves better robustness than other popular AT methods on multiple datasets.
\end{itemize}

\begin{table*}
\caption{Convergence rates for MGDA for (\ref{eq:minimax}) with $n = 1$ and nonconvex-nonconcave $\phi$ . \enquote{---} means no additional assumptions. Here, $K$ is the number of iterations. For MGDA, one iteration consists of one $x$ update and multiple $y$ updates. The optimality metric for the first and third rows is the squared saddle subdifferential norm whereas the optimality metric for the second row is the squared gradient norm of the Moreau envelope of $\max_\delta \phi(\cdot,\delta)$.}
\label{table:MGDA_comp}
\centering
\vspace{1.0em}
\resizebox{\textwidth}{!}{
\begin{tabular}{c c c c c}
\toprule
$f(w)$ & $\phi(w,\delta)$ & $g(\delta)$ & \textbf{Assumption} & \textbf{Convergence Rate} \\ \cmidrule{1-5}
\makecell{Indicator function of a \\ convex compact set} & \makecell{Smooth \\ $-\phi(w,\cdot)$ is $\mu$-PL} & 0 & --- & $O(1/K)$ \citep{nouiehed2019} \\
\cmidrule{1-5}
0 & Lipschitz and smooth & 0 & --- & $O(1/K^{1/2})$ \citep{jin2020} \\
\cmidrule{1-5}
Convex & Smooth & Convex & Weak MVI has solution & $O(1/K)$ (Ours) \\
\bottomrule
\end{tabular}}
\end{table*}

\section{RELATED WORK}

\textbf{Stochastic Primal-Dual Coordinate Methods (SPDCMs).} When $\phi$ is bilinear, the template (\ref{eq:minimax}) encompasses problems such as total variation regularized imaging \citep{chambolle2018} and regularized empirical risk minimization \citep{shwartz2013}. SPDCMs are often used to solve such problems \citep{chambolle2018,fercoq2019,latafat2019,alacaoglu2020}. Due to stochastic coordinate updates in the variable $\delta$, these methods have lower per-iteration costs than deterministic primal-dual hybrid gradient methods \citep{chambolle2011,condat2013,vu2013}. SSI-HG is inspired by a SPDCM called SPDHG, and it generalizes SPDHG to the setting where the coupling function $\phi$ is nonconvex-nonconcave.

\textbf{MGDA.} MGDA \citep{nouiehed2019,thekumparampil2019,barazandeh2020,jin2020} alternates between one gradient descent step for $x$ and multiple gradient ascent steps for $y$ to solve minimax problems. There are two convergence rates for MGDA in the nonconvex-nonconcave setting. \citet{nouiehed2019} proves the rate $O(1/K)$ with respect to the squared saddle subdifferential norm assuming $\phi$ is smooth, $-\phi(w,\cdot)$ is $\mu$-Polyak-{\L}ojasiewicz ($\mu$-PL), $f$ is an indicator function of some convex compact set, and $g \equiv 0$. \citet{jin2020} proves the rate $O(1/K^{1/2})$ with respect to the squared gradient norm of the Moreau envelope of $\max_\delta \phi(\cdot,\delta)$ assuming $f \equiv g \equiv 0$ and $\phi$ is Lipschitz and smooth. See Section 4 of the work by \citet{jin2020} for an explanation on the relation between a function and its Moreau envelope. SI-HGs can be interpreted as variants of MGDA, and our work improves upon previous convergence results for MGDA in the nonconvex-nonconcave setting (Table \ref{table:MGDA_comp}).

\textbf{Convergence Guarantees for AT Methods.} PGD AT \citep{madry2018} alternates between one gradient descent step for $w$ and multiple projected (sign) gradient ascent steps for $\delta$. Hence, PGD AT is essentially MGDA applied to (\ref{eq:AT_intro}). However, we cannot apply the convergence results by \citet{nouiehed2019} or \citet{jin2020} for MGDA to PGD AT. This is because $g$ is an indicator function in the AT setting (c.f. Equation \eqref{eq:AT}). \citet{madry2018} shows the global convergence of PGD AT in the deterministic setting via Danskin's Theorem. There are also works which show the convergence of PGD AT in terms of the loss function value \citep{xing2021,gao2019,zhang2020}.

There are two other AT methods which guarantee convergence to a stationary point of \eqref{eq:AT_intro}. They are variants of PGD AT. The first method, DAT, uses a criterion called first-order stationary condition to adaptively control the number of ascent steps in the inner loop \citep{wang2019}. The second method, YOPO, exploits the compositional structure of DNNs to reduce the computational cost of the inner loop \citep{zhang2019,seidman2020}.

Both DAT and YOPO possess a $O(1/K^{1/2})$ rate of convergence under smoothness and locally strongly concave assumptions. In this work, we prove a $O(1/K)$ convergence rate for SSI-HG under smoothness and weak MVI assumptions.



\section{PRELIMINARIES} \label{sec:prelim}

\textbf{Notations.} We define the saddle subdifferential operator of (\ref{eq:minimax}) as
\begin{align*}
F(w,\delta) =
\begin{bmatrix}
\partial f(w) + \nabla_w \phi(w,\delta) \\
\partial g(\delta) -\nabla_\delta \phi(w,\delta)
\end{bmatrix}
\end{align*}
and its norm as
\begin{align*}
\|F(w,\delta)\| = \inf_{\gamma_f, \gamma_g}
\left\| \begin{bmatrix}
\gamma_f + \nabla_w \phi(w,\delta) \\
\gamma_g -\nabla_\delta \phi(w,\delta)
\end{bmatrix} \right\|
\end{align*}
where the infimum is taken over $\gamma_f \in \partial f(w)$ and $\gamma_g \in \partial g(\delta)$. For a positive scalar $\eta$, we denote $\|\cdot\|^2_\eta = \eta \|\cdot\|^2$. Proximal operator with some function $h(z)$ and $\eta > 0$ is defined as
\begin{align*}
\prox_h^\eta(z) = \argmin_u h(u) + \frac{1}{2} \|u - z\|^2_{\eta^{-1}}.
\end{align*}
We also use the notation $[n] = \{1, \ldots, n\}$. Given a set $\mathcal{S}$, $\I_\mathcal{S}(z)$ is the indicator function which is $0$ on $z \in \mathcal{S}$ and $\infty$ on $z \notin \mathcal{S}$. $\Pi_{\mathcal{S}}[z]$ denotes the projection of $z$ onto $\mathcal{S}$ w.r.t. the Euclidean norm.

\textbf{Nonconvex-nonconcave Minimax Optimization.} We are interested in finding a first-order stationary point of (\ref{eq:minimax}), i.e., a point $(w,\delta)$ which satisfies
\begin{align*}
\mathbf{0} \in F(w,\delta) \quad \text{or equivalently,} \quad \|F(w,\delta)\| = 0.
\end{align*}
We first make the following common assumption.
\begin{assumption} \label{assump:1}
(a) $f$ and $g_i$ are closed, convex, and proper, \\
(b) $\nabla \phi$ is Lipschitz continuous, i.e., there are $L_{11}, L_{12}, L_{22} > 0$ such that
\begin{align*}
\| \nabla_w \phi(w,\delta) - \nabla_w \phi(\bar{w},\delta) \| &\leq L_{11} \|w - \bar{w}\|, \\
\| \nabla_\delta \phi(w,\delta) - \nabla_\delta \phi(\bar{w},\delta) \| &\leq L_{12} \|w - \bar{w}\|, \\
\| \nabla_w \phi(w,\delta) - \nabla_w \phi(w,\bar{\delta}) \| &\leq L_{12} \|\delta - \bar{\delta}\|, \\
\| \nabla_\delta \phi(w,\delta) - \nabla_\delta \phi(w,\bar{\delta}) \| &\leq L_{22} \|\delta - \bar{\delta}\|.
\end{align*}
\end{assumption}
However, even with Assumption \ref{assump:1}, finding a stationary point of (\ref{eq:minimax}) is, in general, intractable \citep{wmvi2021}. Thus, we need to impose additional structures onto the problem. One possible structure is the assumption that there is a solution $(w^*,\delta^*)$ to the MVI problem
\begin{align*}
\begin{bmatrix}
\gamma_f + \nabla_w \phi(w,\delta) \\
\gamma_g -\nabla_\delta \phi(w,\delta)
\end{bmatrix}^\top
\begin{bmatrix}
w - w^* \\
\delta - \delta^*
\end{bmatrix} \geq 0
\end{align*}
$\forall (w,\delta)$, $\forall \gamma_f \in \partial f(w)$, $\forall \gamma_g \in \partial g(\delta)$.
The MVI assumption has been studied by numerous works \citep{dang2014,malitsky2019,liu2020}, and algorithms with good convergence guarantees under the MVI assumption have shown better performance than previous algorithms in training Generative Adversarial Nets as well \citep{gidel2019,mert2019,liu2020gan}. In this work, we consider the following variants of the MVI assumption:

\begin{assumption} \label{assump:2}
There is a solution $(w^*,\delta^*)$ to the weak MVI problem
\begin{align*}
\begin{bmatrix}
\gamma_f + \nabla_w \phi(w,\delta) \\
\gamma_g -\nabla_\delta \phi(w,\delta)
\end{bmatrix}^\top
\begin{bmatrix}
w - w^* \\
\delta - \delta^*
\end{bmatrix} \\
\geq
-\frac{\rho}{2} \left\|
\begin{bmatrix}
\gamma_f + \nabla_w \phi(w,\delta) \\
\gamma_g - \nabla_\delta \phi(w,\delta)
\end{bmatrix}
\right\|^2
\end{align*}
for some $\rho > 0$, $\forall (w,\delta)$, $\forall \gamma_f \in \partial f(w)$, $\forall \gamma_g \in \partial g(\delta)$.
\end{assumption}

\begin{assumption} \label{assump:3}
There is a solution $(w^*,\delta^*)$ to the strong MVI problem
\begin{align*}
\begin{bmatrix}
\gamma_f + \nabla_w \phi(w,\delta) \\
\gamma_g -\nabla_\delta \phi(w,\delta)
\end{bmatrix}^\top
\begin{bmatrix}
w - w^* \\
\delta - \delta^*
\end{bmatrix} \geq
\frac{\mu}{2} \left\|
\begin{bmatrix}
w - w^*\\
\delta - \delta^*
\end{bmatrix}
\right\|^2
\end{align*}
for some $\mu > 0$, $\forall (w,\delta)$, $\forall \gamma_f \in \partial f(w)$, $\forall \gamma_g \in \partial g(\delta)$.
\end{assumption}

We will discuss the justification of these theoretical assumptions in Section \ref{sec:limitation}.

\section{SEMI-IMPLICIT HYBRID GRADIENT METHODS (SI-HGs)} \label{sec:SI-HG}

\textbf{Stochastic SI-HG (SSI-HG).} We propose SSI-HG (Algorithm \ref{alg:SSI-HG}) to find a stationary point of (\ref{eq:minimax}). SSI-HG alternates between a hybrid gradient ascent step on $w$ (line 4) and a stochastic implicit step on $\delta$ (line 6). We remark that when $\phi(w,\delta)$ is bilinear, the implicit step becomes explicit, and SSI-HG reduces to SPDHG with uniform sampling probability\footnote{Let $\phi(w,\delta) = \langle A w, \delta \rangle$ for some matrix $A$. Then $\nabla_w \phi(w,\delta) = A^\top \delta$ and $\nabla_\delta \phi(w,\delta) = A w$. Plug these relations into Algorithm 1, and compare with Algorithm 1 in the paper for SPDHG \citep{chambolle2018}.}. Hence, SSI-HG generalizes SPDHG to the nonconvex-nonconcave minimax optimization setting.

\begin{algorithm}[ht]
\caption{\texttt{SSI-HG}}
\label{alg:SSI-HG}
\begin{algorithmic}[1]
\State \textbf{Input:} $(w^{-1},\delta^{-1}) = (w^0,\delta^0)$, $\sigma$, $\tau$, $\theta$.
\For {$k = 0, 1, 2, \ldots$}
\State $q^k = \nabla_w \phi(w^{k - 1},\delta^{k - 1}) - (n - 1)\{ \nabla_w \phi(w^k,\delta^k) - \nabla_w \phi(w^k,\delta^{k - 1}) \}$
\State $w^{k + 1} = \prox^{\sigma}_f [w^k - \sigma \{ \nabla_w \phi(w^k,\delta^k) + \theta (\nabla_w \phi(w^k,\delta^k) - q^k) \}]$
\State Draw $i_k \in [n]$ uniformly at random.
\State $\delta^{k + 1}_{i_k} = \prox^{\tau}_{g_{i_k}} [\delta^k_{i_k} + \tau \nabla_{\delta_{i_k}} \phi_{i_k}(w^{k + 1},\delta^{k + 1}_{i_k})]$ and $\delta^{k + 1}_i = \delta^k_i$ for all $i \neq i_k$
\EndFor
\end{algorithmic}
\end{algorithm}

We can also interpret SSI-HG as stochastic MGDA. The implicit step is equivalent to
\begin{align} \label{eq:implicit}
\delta^{k + 1}_{i_k} = \argmin_{\delta_{i_k}} g_{i_k}(\delta_{i_k}) &- \phi_{i_k}(w^{k + 1},\delta_{i_k}) \nonumber \\
&+ \frac{1}{2\tau} \|\delta_{i_k} - \delta^k_{i_k}\|^2.
\end{align}
The equivalence is proven in Appendix \ref{append:equiv}. If we assume each $\phi_i$ is smooth and use an iterative proximal method such as FISTA \citep{fista} to solve (\ref{eq:implicit}), SSI-HG alternates between gradient descent on $w$ and multi-step stochastic gradient ascent on $\delta$. Hence, SSI-HG becomes stochastic MGDA. We now present the main results of our paper. The proofs are deferred to Appendices \ref{append:thm1_proof} and \ref{append:thm2_proof}.

\begin{theorem} \label{thm:SSI-HG1}
Suppose Assumptions \ref{assump:1} and \ref{assump:2} are true. Let $\{(w^k,\delta^k)\}$ be the sequence generated by SSI-HG, and define the full-dimensional update (which only depends on $w^k$ and $\delta^{k - 1}$)
\begin{align*}
\hat{\delta}^k = \prox^\tau_g [\delta^{k - 1} + \tau \nabla_\delta \phi(w^k,\hat{\delta}^k)].
\end{align*}
Let $L = \max\{L_{11}, L_{12}, L_{22}\}$. If $\theta = 1$ and
\begin{gather*}
0 < \sigma \leq \frac{1}{6L}, \qquad 0 < \tau \leq \frac{1}{6nL}, \\
0 < \rho < \frac{1}{6 \max\{\sigma^{-1} + 4L^2\sigma, \tau^{-1} + 12 n L^2 \tau\}},
\end{gather*}
we have
\begin{align*}
\frac{1}{K} \sum_{k = 1}^K \E \|F(w^k,\hat{\delta}^k)\|^2 = O(1/K).
\end{align*}
\end{theorem}

\begin{theorem} \label{thm:SSI-HG2}
Suppose Assumptions \ref{assump:1} and \ref{assump:3} are true. Let $\{(w^k,\delta^k)\}$ be the sequence generated by SSI-HG. Let $L = \max\{L_{11}, L_{12}, L_{22}\}$. If $\theta$, $\sigma$, $\tau$ satisfy
\begin{gather*}
0 < \sigma \leq \frac{1}{3L}, \qquad 0 < \tau \leq \frac{1}{3nL}, \\
\theta = \max\left\{ \frac{1}{1 + \mu \sigma}, \frac{1 + (n - 1) \mu \tau / n}{1 + \mu \tau} \right\},
\end{gather*}
we have
\begin{align*}
\E \|w^* - w^K\|^2 = O(\theta^K), \qquad \E \|\delta^* - \delta^K\|^2 = O(\theta^K).
\end{align*}
\end{theorem}

Our proofs for SSI-HG are inspired by those in the work of \citet{alacaoglu2020}. Specifically, Assumption \ref{assump:2} or \ref{assump:3} allows us to characterize the one-iteration behavior of SSI-HG in the nonconvex-nonconcave scenario (Lemma \ref{lemma:one-iter} in Appendix \ref{append:proofs}). We then use telescoping or induction to establish Theorems \ref{thm:SSI-HG1} and \ref{thm:SSI-HG2}.

Since the AT problem (\ref{eq:AT_intro}) is a special case of (\ref{eq:minimax}) (c.f. Equation (\ref{eq:AT})), Theorem \ref{thm:SSI-HG1} improves upon the convergence rates for DAT and YOPO (Table \ref{table:AT_comp}). We have simplified the parameter conditions in Theorems \ref{thm:SSI-HG1} and \ref{thm:SSI-HG2} for readability. The general forms are written in Appendices \ref{append:thm1_proof} and \ref{append:thm2_proof}.

\textbf{Deterministic SI-HG (DSI-HG).} When $n = 1$, SSI-HG becomes DSI-HG (Algorithm \ref{alg:DSI-HG}). All the convergence results in Theorems \ref{thm:SSI-HG1} and \ref{thm:SSI-HG2} hold for DSI-HG with expectation removed (Corollaries \ref{cor:DSI-HG1} and \ref{cor:DSI-HG2}). In particular, if we interpret DSI-HG as MGDA, Corollary \ref{cor:DSI-HG1} improves upon previous guarantees for MGDA in the nonconvex-nonconcave setting (Table \ref{table:MGDA_comp}).

\begin{algorithm}[H]
\caption{\texttt{DSI-HG}}
\label{alg:DSI-HG}
\begin{algorithmic}[1]
\State \textbf{Input:} $(w^{-1},\delta^{-1}) = (w^0,\delta^0)$, $\sigma$, $\tau$, $\theta$.
\For {$k = 0, 1, 2, \ldots$}
\State $w^{k + 1} = \prox^{\sigma}_f [w^k - \sigma \{ \nabla_w \phi(w^k,\delta^k) + \theta (\nabla_w \phi(w^k,\delta^k) - \nabla_w \phi(w^{k - 1},\delta^{k - 1})) \}]$
\State $\delta^{k + 1} = \prox^{\tau}_{g} [\delta^k + \tau \nabla_{\delta} \phi(w^{k + 1},\delta^{k + 1})]$
\EndFor
\end{algorithmic}
\end{algorithm}

\begin{corollary} \label{cor:DSI-HG1}
Suppose Assumptions \ref{assump:1} and \ref{assump:2} are true. Let $\{(w^k,\delta^k)\}$ be the sequence generated by DSI-HG. Let $L = \max\{L_{11}, L_{12}, L_{22}\}$. If $\theta = 1$ and
\begin{gather*}
0 < \sigma, \tau \leq \frac{1}{6L}, \\
0 < \rho < \frac{1}{6 \max\{\sigma^{-1} + 4L^2\sigma, \tau^{-1} + 12L^2 \tau\}},
\end{gather*}
we have
\begin{align*}
\frac{1}{K} \sum_{k = 1}^K \|F(w^k,\delta^k)\|^2 = O(1/K).
\end{align*}
\end{corollary}

\begin{corollary} \label{cor:DSI-HG2}
Suppose Assumptions \ref{assump:1} and \ref{assump:3} are true. Let $\{(w^k,\delta^k)\}$ be the sequence generated by DSI-HG. Let $L = \max\{L_{11}, L_{12}, L_{22}\}$. If $\theta$, $\sigma$, $\tau$ satisfy
\begin{align*}
 0 < \sigma, \tau \leq \frac{1}{3L}, \qquad \theta = \max\left\{ \frac{1}{1 + \mu \sigma}, \frac{1}{1 + \mu \tau} \right\},
\end{align*}
we have
\begin{align*}
\|w^* - w^K\|^2 = O(\theta^K), \qquad \|\delta^* - \delta^K\|^2 = O(\theta^K).
\end{align*}
\end{corollary}

We again remark that we have simplified the parameter conditions for readability of the Corollaries. The general forms are written in Appendix \ref{append:corollaries}.

\subsection{Applying SI-HGs to AT} \label{sec:SI-HG-AT}

We now extend our theoretical intuition to the AT setting. This section is inspired by works which combine theoretically established algorithms with practical algorithms or use heuristics to extend algorithms to deep learning settings \citep{das2018,gidel2019,mert2019,nouiehed2019,chav2020,sinha2018,wang2019}.

Denote $\B_i \coloneqq \{\delta_i' : \|\delta_i'\|_\infty \leq \epsilon\}$. The AT problem (\ref{eq:AT_intro}) can be written as
\begin{align} \label{eq:AT}
\min_w \max_{\delta} \sum_{i = 1}^n \phi_i(w,\delta_i) - \I_{\B_i}(\delta_i)
\end{align}
where
\begin{align*}
\phi_i(w,\delta_i) = \frac{1}{n} \ell(x_i + \delta_i, y_i, w).
\end{align*}
When the dimension of $\delta$ or $w$ is large, i.e., the dataset is large or the DNN has a lot of parameters, it may be difficult to directly apply SSI-HG or DSI-HG to the AT problem. Hence, we propose minibatch SI-HG (MSI-HG, Algorithm \ref{alg:MSI-HG} in Appendix \ref{append:pseudocodes}), which is a combination of the minibatch gradient method \citep{bottou2016} and DSI-HG. We also develop a heuristic for solving the implicit step in SI-HGs.



\begin{table*}[h!]
\caption{Accuracy (\%) on natural and adversarial examples at the final iteration.}
\label{table:exp_stat}
\centering
\resizebox{\textwidth}{!}{
\begin{tabular}{c c c c c c c c c c}
\toprule
\multirow{2}{4em}{\textbf{Method}} & \multicolumn{3}{c}{\textbf{MNIST}} & \multicolumn{3}{c}{\textbf{SVHN}} & \multicolumn{3}{c}{\textbf{CIFAR-10}} \\
\cmidrule(lr){2-4} \cmidrule(lr){5-7} \cmidrule(lr){8-10}
 & Natural & PGD-20 & PGD-50-10 & Natural & PGD-20 & PGD-50-10 & Natural & PGD-20 & PGD-50-10 \\
\cmidrule{1-10}
PGD AT & $94.68_{\pm 0.21}$ & $54.24_{\pm 4.27}$ & $44.30_{\pm 4.60}$ & $91.31_{\pm 0.55}$ & $66.94_{\pm 0.80}$ & $66.40_{\pm 0.83}$ & $76.19_{\pm 0.25}$ & $46.35_{\pm 0.49}$ & $45.84_{\pm 0.46}$ \\
DAT & $93.27_{\pm 0.71}$ & $16.73_{\pm 6.17}$ & $10.00_{\pm 3.87}$ & $91.23_{\pm 0.38}$ & $61.68_{\pm 0.46}$ & $61.04_{\pm 0.52}$ & $67.93_{\pm 0.39}$ & $34.72_{\pm 0.18}$ & $34.18_{\pm 0.18}$ \\ 
YOPO & $\mathbf{97.73}_{\pm 0.09}$ & $24.92_{\pm 12.29}$ & $14.63_{\pm 10.09}$ & $89.87_{\pm 0.52}$ & $46.35_{\pm 2.20}$ & $44.71_{\pm 2.24}$  & \textbf{83.99} & 44.72 & --- \\
\cmidrule{1-10}
MSI-HG (Ours) & $94.89_{\pm 0.63}$ & $\mathbf{62.34_{\pm 3.33}}$ & $\mathbf{44.51_{\pm 2.04}}$ & $\mathbf{92.52_{\pm 0.11}}$ & $\mathbf{68.44_{\pm 0.15}}$ & $\mathbf{67.86_{\pm 0.15}}$ & $82.04_{\pm 0.13}$ & $\mathbf{48.91_{\pm 0.23}}$ & $\mathbf{48.27_{\pm 0.20}}$ \\
\bottomrule
\end{tabular}}
\end{table*}

\begin{table*}[h!]
\caption{Accuracy (\%) on natural and adversarial examples at the moment of best robust accuracy. We do not report accuracies on adversarial examples generated by PGD-50-10, as it was too expensive to run PGD-50-10 at every iteration of training.}
\label{table:exp_best_stat}
\centering
\resizebox{0.8\textwidth}{!}{
\begin{tabular}{c c c c c c c}
\toprule
\multirow{2}{4em}{\textbf{Method}} & \multicolumn{2}{c}{\textbf{MNIST}} & \multicolumn{2}{c}{\textbf{SVHN}} & \multicolumn{2}{c}{\textbf{CIFAR-10}} \\
\cmidrule(lr){2-3} \cmidrule(lr){4-5} \cmidrule(lr){6-7}
& Natural & PGD-20 & Natural & PGD-20 & Natural & PGD-20 \\
\cmidrule{1-7}
PGD AT & $94.65_{\pm 0.22}$ & $54.71_{\pm 4.23}$ & $91.09_{\pm 0.50}$ & $67.21_{\pm 0.88}$ & $75.83_{\pm 0.36}$ & $46.73_{\pm 0.21}$ \\
DAT & $90.96_{\pm 4.88}$ & $28.58_{\pm 4.04}$ & $91.08_{\pm 0.21}$ & $61.94_{\pm 0.26}$ & $67.83_{\pm 0.35}$ & $34.85_{\pm 0.11}$ \\
YOPO & $\mathbf{97.66}_{\pm 0.05}$ & $37.23_{\pm 2.48}$ & $89.54_{\pm 0.47}$ & $48.19_{\pm 0.84}$ & --- & --- \\
\cmidrule{1-7}
MSI-HG (Ours) & $94.95_{\pm 0.44}$ & $\mathbf{66.31_{\pm 3.41}}$ & $\mathbf{92.52_{\pm 0.11}}$ & $\mathbf{68.44_{\pm 0.15}}$ & $\mathbf{81.63_{\pm 0.23}}$ & $\mathbf{49.30_{\pm 0.11}}$ \\ 
\bottomrule
\end{tabular}}
\end{table*}

Under the AT setting (\ref{eq:AT}), the minimization form of the implicit step (\ref{eq:implicit}) becomes
\begin{align} \label{eq:AT_implicit}
\delta^{k + 1}_{i_k} = \argmin_{\delta_{i_k} \in \B_{i_k}} -\phi_{i_k}(w^{k + 1},\delta_{i_k}) + \frac{1}{2\tau} \|\delta_{i_k} - \delta^k_{i_k}\|^2.
\end{align}
Following the intuition behind Section 4 of the work by Goodfellow et al. \citep{goodfellow2015}, we would like to use sign gradient to update $\delta$.
Sign gradient is used in many other AT methods as well \citep{wang2019,zhang2019,madry2018,trades}.
However, the quadratic penalty in (\ref{eq:AT_implicit}) may be incompatible with sign gradient. To circumvent this problem, we interpret (\ref{eq:AT_implicit}) as searching for an adversarial perturbation within the proximity of $\delta^k_{i_k}$. Previous works use the infinity norm to measure the distance between perturbation vectors \citep{yao2019,pooladian2020}. Hence, we solve the surrogate problem
\begin{align}
\delta^{k + 1}_{i_k} = \argmin_{\delta_{i_k} \in \S_{i_k}} -\phi_{i_k}(w^{k + 1},\delta_{i_k})
\end{align}
where
\begin{align}
\S_{i_k} = \B_{i_k} \cap \{\delta_{i_k} : \| \delta_{i_k} - \delta^k_{i_k} \|_\infty \leq \tau \}
\end{align}
with $T$ steps of PGD:
\begin{align}
\delta^{k,0}_{i_k} = \delta^k_{i_k} + \unif[-\tau,\tau]^{d_{i_k}}
\end{align}
and
\begin{align} \label{eq:AT_pgd}
\delta^{k,t + 1}_{i_k} = \Pi_{\S_{i_k}} [ \delta^{k,t}_{i_k} + \eta \cdot \sign(\nabla_{\delta_{i_k}} \phi_{i_k}(w^{k + 1},\delta^{k,t}_{i_k})) ]
\end{align}
for $t = 0, \ldots, T - 1$, and $\delta^{k + 1}_{i_k} = \delta^{k,T}_{i_k}$. $\tau$ and $\eta$ are hyperparameters.

\section{EXPERIMENTS} \label{sec:exp}

\begin{figure*}[h!]
\centering
\includegraphics[width=\linewidth]{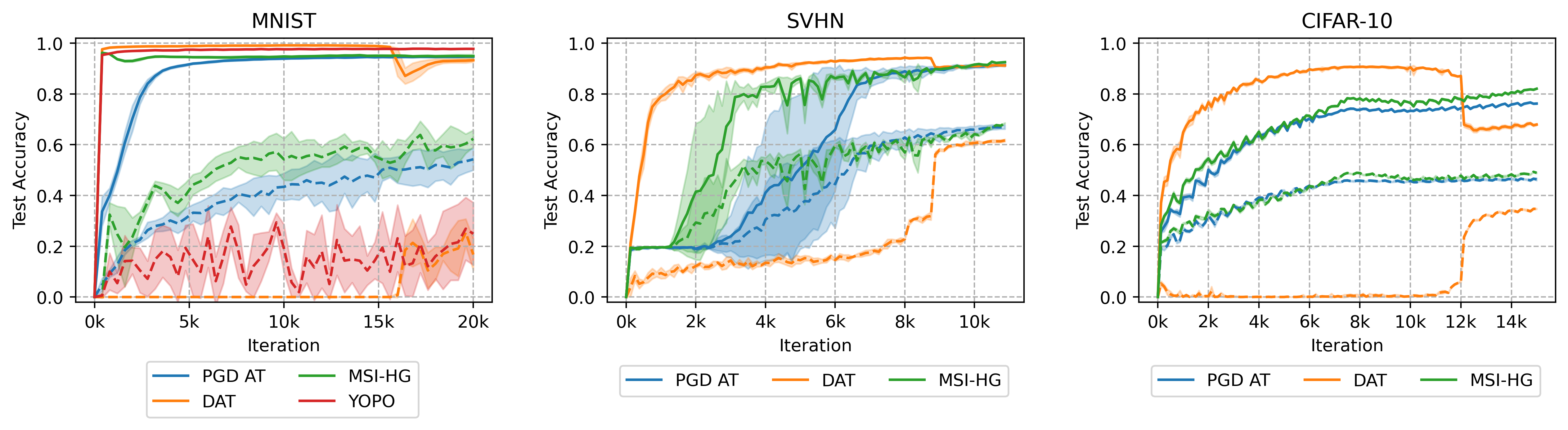}
\includegraphics[width=\linewidth]{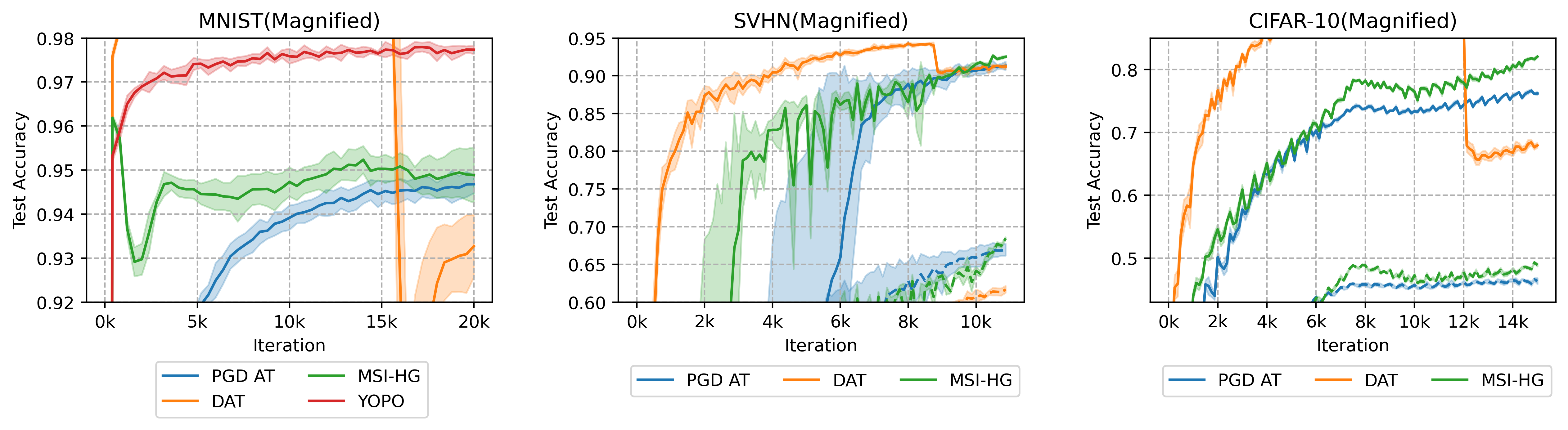}
\caption{Learning curves of AT algorithms. Solid line denotes natural accuracy and dotted line denotes accuracy on adversarial examples generated by PGD-20. On SVHN, the natural accuracy curves for PGD AT and DAT overlap near the end.}
\label{fig:exp}
\end{figure*}

\begin{figure*}[h!]
\centering
\includegraphics[width=\linewidth]{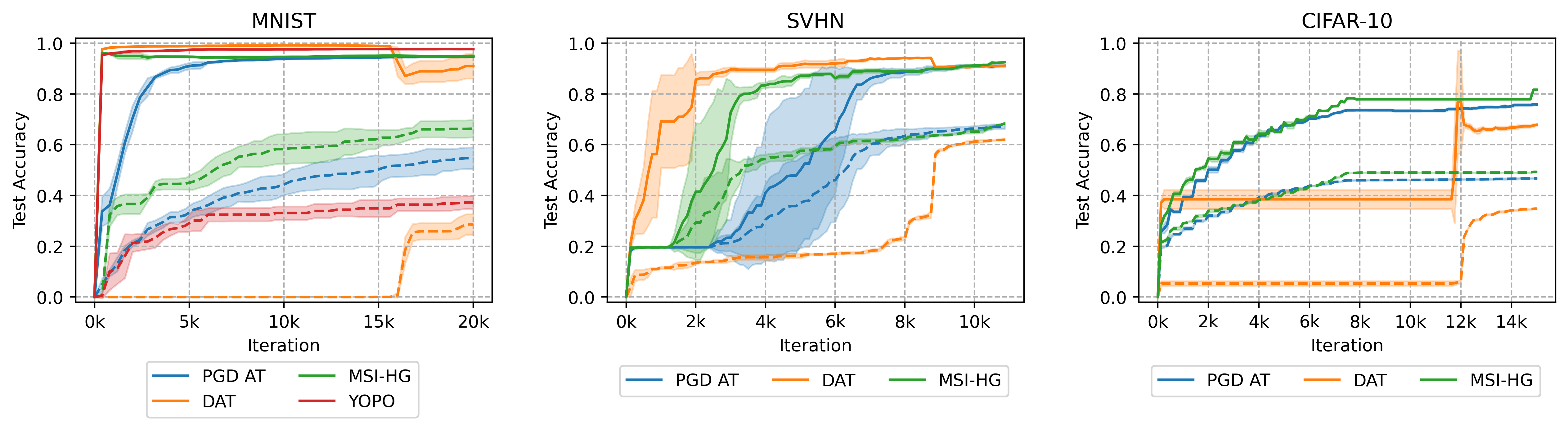}
\includegraphics[width=\linewidth]{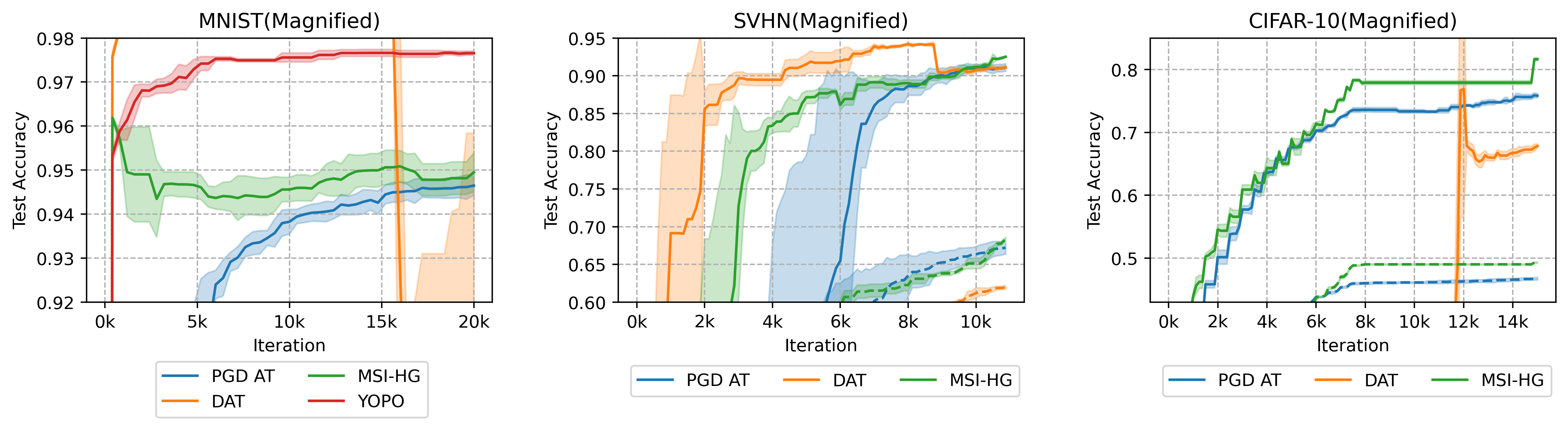}
\caption{Processed learning curves of AT algorithms. At iteration $K$, we plot the natural and robust accuracies of the iteration $\leq K$ which had the best robust accuracy. Solid line denotes natural accuracy and dotted line denotes accuracy on adversarial examples generated by PGD-20. On SVHN, the natural accuracy curves for PGD AT and DAT overlap near the end.}
\label{fig:exp_best}
\end{figure*}

We run AT on  MNIST \citep{mnist}, SVHN \citep{svhn}, and CIFAR-10 \citep{cifar10}. We denote the $T$-step PGD attack with $R$ restarts by PGD-$T$-$R$. If $R = 0$, we omit $R$ from the name. The baseline methods are AT methods which have convergence guarantees: PGD AT, DAT, and YOPO. Following common practice, we set the number of inner iterations in PGD AT to 10 \citep{co2020,wang2019,zhang2019}. For SI-HGs, we set $\theta = 1$ and use the smallest choices of $\tau \in (0,2\epsilon)$ and $T \in \{2, 3, 5, 10\}$ which lead to convergence without harming the robustness. For PGD, DAT, and MSI-HG, the $\delta$ update step size $\eta$ is always set to $2.5\epsilon / T$ so $\delta$ may reach and move around on the boundary of the constraint set. For other hyperparameters, we use the recommended settings. All learning curves and statistics are produced by averaging over five random trials.

\textbf{Training settings.} We use a four-layer DNN (3 conv. layers with channels 16, 32, 64, and a final linear layer) on MNIST with $\epsilon = 0.4$, PreActResNet-8 \citep{he2016} for SVHN with $\epsilon = 4 / 255$, and PreActResNet-18 on CIFAR-10 with $\epsilon = 8 / 255$. We use a single A100-SXM4-40GB to train each model. We use batch size 150 on MNIST and batch size 100 on SVHN and CIFAR-10. On each dataset, we combine each AT algorithm with SGD with momentum 0.9 plot its learning curve\footnote{In Appendix \ref{append:pseudocodes}, we show how MSI-HG is combined with GD or SGD with momentum.}. On MNIST, each methods is run for 50 epochs. SGD uses a constant learning rate 0.01. On SVHN, each method is run for 15 epochs. SGD uses a triangular learning rate \citep{smith2017} which starts at zero, peaks at epoch 5 with value 0.2, and decays back to zero\footnote{We remark that triangular learning rates have been used in recent works to reduce the training time of AT by up to hundred orders of magnitude \citep{co2020,wong2020}.}.
On CIFAR-10, each method is run for 30 epochs. SGD uses a cyclic learning rate which starts at zero, peaks at epoch 5 with value 0.2, decays to zero at epoch 15, peaks at epoch 20 with value 0.02, and decays to zero at epoch 30.
We found that YOPO did not converge under the triangular and cyclic learning rate schedules. Hence, we follow the settings recommended by the authors for YOPO on SVHN experiments. On CIFAR-10, we reuse the best accuracies reported by the authors of YOPO (c.f. Table 2 in \citep{zhang2019}). Further training settings are described in Appendix \ref{append:settings}.

\textbf{Experiment results.} In Figure \ref{fig:exp}, we plot the learning curves of AT algorithms.
On all datasets, we see acceleration for MSI-HG.
For YOPO and DAT, the natural accuracies rise faster than those of MSI-HG, but robust accuracies grow extremely slowly. This is undesirable, since our goal is \textit{faster convergence to adversarially robust DNNs}. On MNIST and SVHN, MSI-HG shows acceleration during the first half of the training process. If we examine the magnified learning curves on CIFAR-10, MSI-HG shows faster convergence than PGD AT at iterations 6k to 8k and 10k to 15k. As a result, the accuracies of PGD AT at epoch 30 is lower than those of MSI-HG at epoch 15. If we process the learning curves to reduce oscillations (Figure \ref{fig:exp_best}), acceleration for MSI-HG is even more evident. We note that MSI-HG uses at most ten steps in the inner loop, so its computational cost is similar to or better than that of PGD AT.

We also report natural and robust accuracies in Tables \ref{table:exp_stat} and \ref{table:exp_best_stat}. MSI-HG achieve better robustness than the baseline methods on all datasets\footnote{On CIFAR-10, to generate PGD-20 adversarial examples, the authors of YOPO use $\eta = 2 / 255$ while we use $\eta = 1 / 255$. However, PGD-50-10 accuracy for MSI-HG is higher than PGD-20 accuracy for YOPO. Thus, PreActResNet-18 trained by MSI-HG is indeed more robust than those trained by YOPO.}. YOPO achieves higher natural accuracy at the cost of lower robustness. On MNIST, robustness for YOPO is especially poor. MSI-HG beat PGD AT and DAT in terms of both natural and robust accuracies.

\section{A DISCUSSION ON THE LIMITATIONS OF OUR WORK} \label{sec:limitation}

\subsection{Theoretical Assumptions}

Here, we discuss and attempt to justify the theoretical assumptions used in our work.

\textbf{Assumption \ref{assump:1}.} Assumption \ref{assump:1} is generally false in the deep learning setting. For instance, the assumption that the coupling function $\phi$ has Lipschitz continuous gradients is false when we train DNNs which use non-differentiable operations such as ReLU or max-pooling. Still, many works show that algorithms with such theoretical results perform surprisingly well in the deep learning setting \citep{das2018,gidel2019,mert2019,nouiehed2019,chav2020,wang2019,attention}. We speculate that this is because the assumptions hold approximately or locally when we train DNNs\footnote{Wang et al. \citep{wang2019} also considers this perspective.}. For example, a recent work has shown that semi-smoothness, an approximate version of smoothness, holds for overparametrized ReLU DNNs \citep{zhu2018}. Thus, even if we have to work under some restrictive assumptions, it is crucial that we continue to develop theoretical grounds for AT methods.

\textbf{Assumption \ref{assump:2}.} The relation between MVI and AT is a non-trivial research topic by itself, but we try our best to justify the MVI condition here. First, the MVI condition is already weaker than other assumptions such as pseudo-monotonicity, monotonicity, or coherence \citep{mert2019}. In fact, we use the even weaker weak MVI condition. Second, algorithms developed under the MVI condition have shown good performance when applied to deep learning \citep{gidel2019,mert2019,liu2020gan}. Finally, \citet{liu2020gan} has pointed out that the MVI condition holds while using SGD to learn neural nets for minimization. As nonconvex-nonconcave minimax optimization is in general intractable \citep{wmvi2021}, we believe the MVI condition is an adequate choice to develop new AT algorithms in a principled manner.

\subsection{Per-Iteration Costs}

Due to the full-gradient update in $w$, the per-iteration cost for SSI-HG is larger than per-iteration costs of doubly-stochastic methods. Hence, it is necessary to establish theoretical results for methods such as MSI-HG. Based on the experiments, we cautiously conjecture that MSI-HG also has a rate better than $O(1/K^{1/2})$. It may be possible to use ideas from our proofs for SSI-HG and DSI-HG to prove the convergence of MSI-HG, but we leave this for future work.

\subsection{Memorization of $\delta$}

Although we introduced hybrid gradient methods and PGD AT / DAT / YOPO as the same class of algorithm (MGDA), they are different in the aspect that PGD AT / DAT / YOPO does not memorize $\delta$ but hybrid gradient methods do. This is because PGD AT / DAT / YOPO randomly initializes $\delta$ at every iteration. This may seem like a drawback of the hybrid gradient methods, yet this memory of $\delta$ is what allows us to apply the momentum-like update (line 3 of Algorithms \ref{alg:SSI-HG} and \ref{alg:DSI-HG}) and thus obtain better convergence rates. This intuition is reflected in the faster convergence of MSI-HG in the adversarial training (AT) experiments.

\section{CONCLUSION} \label{sec:conclusion}

In this work, we introduced SI-HGs to solve nonconvex-nonconcave minimax problems separable in the maximization variable. We proved that SI-HGs achieve the convergence rate $O(1/K)$ which improves upon the convergence rate $O(1/K^{1/2})$ of YOPO and DAT. Our work also improved upon previous convergence results for MGDA. Experiments showed that a practical variant of SI-HGs indeed converges faster and achieves better robustness than other AT methods. Finally, we discussed the limitations of our work, and proposed future directions of research.

We generally expect positive outcomes from this work, since robustness and efficiency are desirable properties of machine learning systems. Adversarially robust DNNs are more likely to be resistant to malicious input manipulations than their naturally trained counterparts. Algorithms which converge fast consume less resource than other algorithms with similar per-iteration costs.

{\small

\bibliography{references}
\bibliographystyle{bst}}


\clearpage
\appendix

\thispagestyle{empty}

\onecolumn \makesupplementtitle

\section{PSEUDOCODES} \label{append:pseudocodes}



\begin{algorithm}[H]
\caption{\texttt{MSI-HG}}
\label{alg:MSI-HG}
\begin{algorithmic}[1]
\State \textbf{Input:} $(w^{-1},\delta^{-1}) = (w^0,\delta^0)$, $\sigma$, $\tau$.
\For {$k = 0, 1, 2, \ldots$}
\State{Jointly shuffle the entries of $\delta$ and $\phi$.}
\For {$i = 1, \ldots, n$}
\State $\delta^{nk + i}_i = \Pi_{\B_i} [\delta^{nk + i - 1}_i + \tau \nabla_{\delta_i} \phi_i(w^{nk + i - 1},\delta^{nk + i}_i)]$
\State $\delta^{nk + i}_j = \delta^{nk + i - 1}_j$ for all $j \neq i$
\State $w^{nk + i} = w^{nk + i - 1} - \sigma \{ 2 \nabla_w \phi_i(w^{nk + i - 1},\delta^{nk + i}_i) - \nabla_w \phi_{i - 1}(w^{nk + i - 2},\delta^{nk + i - 1}_{i - 1}) \}$
\EndFor
\EndFor
\end{algorithmic}
\end{algorithm}

\begin{algorithm}[H]
\caption{\texttt{MSI-HG+GD}}
\label{alg:MSI-HG+GD}
\begin{algorithmic}[1]
\State \textbf{Input:} $(w^{-1},\delta^{-1}) = (w^0,\delta^0)$, $\sigma$, $\tau$, $\rho$.
\For {$k = 0, 1, 2, \ldots$}
\State{Jointly shuffle the entries of $\delta$ and $\phi$.}
\For {$i = 1, \ldots, n$}
\State $\delta^{nk + i}_i = \Pi_{\B_i} [\delta^{nk + i - 1}_i + \tau \nabla_{\delta_i} \phi_i(w^{nk + i - 1},\delta^{nk + i}_i)]$
\State $\delta^{nk + i}_j = \delta^{nk + i - 1}_j$ for all $j \neq i$
\State $\nabla_w^{nk + i - 1} = 2 \nabla_w \phi_i(w^{nk + i - 1},\delta^{nk + i}_i) - \nabla_w \phi_{i - 1}(w^{nk + i - 2},\delta^{nk + i - 1}_{i - 1})$
\State $w^{nk + i} = w^{nk + i - 1} - \sigma \cdot \texttt{GD}[\nabla_w^{nk + i - 1}, \rho, nk + i]$
\EndFor
\EndFor
\end{algorithmic}
\end{algorithm}

We also use $T$ steps of PGD to approximate line 5 in Algorithms \ref{alg:MSI-HG} and \ref{alg:MSI-HG+GD}.

\newpage

\section{MISSING PROOFS} \label{append:proofs}

\subsection{Proof of the equivalence between the implicit step and Equation (\ref{eq:implicit})} \label{append:equiv}

By the definition of the proximal operator, line 6 of SSI-HG (Algorithm \ref{alg:SSI-HG}) is
\begin{align} \label{eq:8}
\delta^{k + 1}_{i_k} = \argmin_{\delta_{i_k}} g_{i_k}(\delta_{i_k}) + \frac{1}{2} \|\delta_{i_k} - \{\delta^k_{i_k} + \tau_k \nabla_{\delta_{i_k}} \phi_{i_k}(w^{k + 1},\delta^{k + 1}_{i_k})\}\|^2_{\tau^{-1}}
\end{align}
The optimality condition of (\ref{eq:8}) is
\begin{align}
0 \in \partial g_{i_k}(\delta^{k + 1}_{i_k}) + \tau^{-1} [\delta^{k + 1}_{i_k} - \{\delta^k_{i_k} + \tau_k \nabla_{\delta_{i_k}} \phi_{i_k}(w^{k + 1},\delta^{k + 1}_{i_k})\}]
\end{align}
which is equivalent to
\begin{align}
0 \in \partial g_{i_k}(\delta^{k + 1}_{i_k}) - \nabla_{\delta_{i_k}} \phi_{i_k}(w^{k + 1},\delta^{k + 1}_{i_k}) + \tau^{-1} (\delta^{k + 1}_{i_k} - \delta^k_{i_k}).
\end{align}
This is the optimality condition for (\ref{eq:implicit}).

\subsection{One-iteration result for SSI-HG}

Define for $k \geq 0$, the filtration and the conditional expectation
\begin{align*}
\F_0 = \emptyset, \qquad \F_k = \{i_0, \ldots, i_{k - 1}\}, \qquad \E_k = \E[ \cdot \mid \F_k].
\end{align*}
We also define the following representation of SSI-HG with full dimensional updates
\begin{align*}
q^k &= \nabla_w \phi(w^{k - 1},\delta^{k - 1}) - (n - 1)\{ \nabla_w \phi(w^k,\delta^k) - \nabla_w \phi(w^k,\delta^{k - 1}) \}, \\
w^{k + 1} &= \prox^{\sigma_k}_f [w^k - \sigma \{ \nabla_w \phi(w^k,\delta^k) + \theta_k (\nabla_w \phi(w^k,\delta^k) - q^k) \} ], \\
\hat{\delta}^{k + 1}_i &= \prox^{\tau}_{g_i} [\delta^k_i + \tau \nabla_{\delta_i} \phi_i(w^{k + 1},\hat{\delta}^{k + 1}_i)].
\end{align*}
The optimality conditions for $w^{k + 1}$ and $\hat{\delta}^{k + 1}_i$ are
\begin{align}
0 &= \gamma^{k + 1}_f + \sigma^{-1}[w^{k + 1} - w^k + \sigma\{ \nabla_w \phi(w^k,\delta^k) + \theta (\nabla_w \phi(w^k,\delta^k) - q^k) \}] \label{eq:opt1} \\
0 &= \hat{\gamma}^{k + 1}_{g_i} + \tau^{-1}(\hat{\delta}^{k + 1}_i - \delta^k_i - \tau \nabla_{\delta_i} \phi_i(w^{k + 1},\hat{\delta}^{k + 1}_i)) \label{eq:opt2}
\end{align}
for some $\gamma_f^{k + 1} \in \partial f(w^{k + 1})$ and $\hat{\gamma}^{k + 1}_{g_i} \in \partial g_i(\hat{\delta}^{k + 1}_i)$. Note that
\begin{align*}
\hat{\gamma}^{k + 1}_g \coloneqq (\hat{\gamma}^{k + 1}_{g_1}, \ldots, \hat{\gamma}^{k + 1}_{g_n}) \in \partial g(\hat{\delta}^{k + 1})
\end{align*}
so from (\ref{eq:opt2}),
\begin{align}
0 = \hat{\gamma}^{k + 1}_g + \tau^{-1}(\hat{\delta}^{k + 1} - \delta^k - \tau \nabla_\delta \phi(w^{k + 1},\hat{\delta}^{k + 1})). \label{eq:opt3}
\end{align}
We also note that the optimality condition for $\delta^{k + 1}_{i_k}$ is
\begin{align}
0 &= \gamma^{k + 1}_{g_{i_k}} + \tau^{-1}(\delta^{k + 1}_{i_k} - \delta^k_{i_k} - \tau \nabla_{\delta_{i_k}} \phi_{i_k}(w^{k + 1},\delta^{k + 1}_{i_k})) \label{eq:opt4}
\end{align}
for some $\gamma^{k + 1}_{g_{i_k}} \in \partial g_{i_k}(\delta^{k + 1}_{i_k})$. We start with some technical Lemmas.

\begin{lemma}
For any $\delta \in \R^{d_1} \times \cdots \times \R^{d_n}$ and any function $r(\delta) = \sum_{i = 1}^n r_i(\delta_i)$,
\begin{align}
r(\hat{\delta}^{k + 1}) &= \E_k r(\delta^{k + 1}) + (n - 1)\{ \E_k r(\delta^{k + 1}) - r(\delta^k) \} \label{eq:e1} \\
\|\delta - \hat{\delta}^{k + 1}\|^2_{\tau^{-1}} &=  n \E_k \|\delta - \delta^{k + 1}\|^2_{\tau^{-1}} - n \|\delta - \delta^k\|^2_{\tau^{-1}} + \|\delta - \delta^k\|^2_{\tau^{-1}}, \label{eq:e2} \\
\|\hat{\delta}^{k + 1} - \delta^k\|^2_{\tau^{-1}} &= n \E_k \|\delta^{k + 1} - \delta^k\|^2_{\tau^{-1}}. \label{eq:e3}
\end{align}
\end{lemma}
\begin{proof}
Let us observe that
\begin{align*}
\E_k r_i(\delta^{k + 1}_i) = \frac{1}{n} r_i(\hat{\delta}^{k + 1}_i) + \left(1 - \frac{1}{n}\right) r_i(\delta^k_i).
\end{align*}
Summing the above equation over $i$, multiplying both sides by $n$, and rearranging the terms yields (\ref{eq:e1}). Using $r(\delta) = \|\delta - \hat{\delta}^{k + 1}\|^2_{\tau^{-1}}$ yields (\ref{eq:e2}). Finally, plugging in $\delta = \delta^k$ into (\ref{eq:e2}) yields (\ref{eq:e3}).
\end{proof}

\newpage

\begin{lemma} \label{lemma:4}
Assume $\sigma > 0$, $\tau > 0$, $\theta \in (0,1]$, and define
\begin{align*}
\kappa = \max\{L_{12} (\sigma \tau n)^{1/2}, L_{11} \sigma\}.
\end{align*}
We then have
\begin{align}
&| \inner{w^k - w^{k + 1}}{\nabla_w \phi(w^k,\delta^k) - q^k} | \nonumber \\
&\leq \kappa \|w^{k + 1} - w^k\|^2_{\sigma^{-1}} + \frac{\kappa}{2} \|w^k - w^{k - 1}\|^2_{\sigma^{-1}} + \frac{n\kappa}{2} \|\delta^k - \delta^{k - 1}\|^2_{\tau^{-1}} \label{eq:bound1} \\
&\leq \frac{\kappa}{\theta} \|w^{k + 1} - w^k\|^2_{\sigma^{-1}} + \frac{\kappa}{2} \|w^k - w^{k - 1}\|^2_{\sigma^{-1}} + \frac{n\kappa}{2} \|\delta^k - \delta^{k - 1}\|^2_{\tau^{-1}}. \label{eq:bound2}
\end{align}
\end{lemma}
\begin{proof}
We note that since $\theta \in (0,1]$, it suffices to prove (\ref{eq:bound1}). By the definition of $q^k$,
\begin{align} \label{eq:lemma4_1}
&\inner{w^k - w^{k + 1}}{ \nabla_w \phi(w^k,\delta^k) - q^k} \nonumber \\
&= \inner{w^k - w^{k + 1}}{ \nabla_w \phi(w^k,\delta^k) + (n - 1)\{ \nabla_w \phi(w^k,\delta^k) - \nabla_w \phi(w^k,\delta^{k - 1}) \} - \nabla_w \phi(w^{k - 1},\delta^{k - 1})} \nonumber \\
&= n \inner{w^k - w^{k + 1}}{\nabla_w \phi(w^k,\delta^k) - \nabla_w \phi(w^k,\delta^{k - 1})} + \inner{w^k - w^{k + 1}}{\nabla_w \phi(w^k,\delta^{k - 1}) - \nabla_w \phi(w^{k - 1},\delta^{k - 1})}.
\end{align}
We now bound the two inner products in (\ref{eq:lemma4_1}). The first inner product can be bounded as
\begin{align} \label{eq:lemma4_2}
&\kappa^{-1} |\inner{w^k - w^{k + 1}}{\nabla_w \phi(w^k,\delta^k) - \nabla_w \phi(w^k,\delta^{k - 1})}| \nonumber \\
&\leq (L_{12}^2 \sigma \tau n)^{-1/2} |\inner{w^k - w^{k + 1}}{\nabla_w \phi(w^k,\delta^k) - \nabla_w \phi(w^k,\delta^{k - 1})}| \nonumber \\
&\leq (L_{12}^2 \sigma \tau n)^{-1/2} \|w^{k + 1} - w^k\| \|\nabla_w \phi(w^k,\delta^k) - \nabla_w \phi(w^k,\delta^{k - 1})\| \nonumber \\
&\leq (\sigma \tau n)^{-1/2} \|w^{k + 1} - w^k\| \|\delta^k - \delta^{k - 1}\| \nonumber \\
&\leq \frac{1}{2n} \|w^{k + 1} - w^k\|^2_{\sigma^{-1}} + \frac{1}{2} \|\delta^k - \delta^{k - 1}\|^2_{\tau^{-1}}
\end{align}
where we have used the definition of $\kappa$ at the first inequality, Cauchy-Schwarz inequality at the second inequality, Lipschitz continuity of $\nabla_w \phi$ at the third inequality, and Young's inequality at the last inequality.

The second inner product can be bounded as
\begin{align} \label{eq:lemma4_3}
&\kappa^{-1} |\inner{w^k - w^{k + 1}}{\nabla_w \phi(w^k,\delta^{k - 1}) - \nabla_w \phi(w^{k - 1},\delta^{k - 1})}| \nonumber \\
&\leq (L_{11}^2 \sigma^2)^{-1/2} |\inner{w^k - w^{k + 1}}{\nabla_w \phi(w^k,\delta^{k - 1}) - \nabla_w \phi(w^{k - 1},\delta^{k - 1})}| \nonumber \\
&\leq (L_{11}^2 \sigma^2)^{-1/2} \|w^{k + 1} - w^k\| \|\nabla_w \phi(w^k,\delta^{k - 1}) - \nabla_w \phi(w^{k - 1},\delta^{k - 1})\| \nonumber \\
&\leq (\sigma^2)^{-1/2} \|w^{k + 1} - w^k\| \|w^k - w^{k - 1}\| \nonumber \\
&\leq \frac{1}{2} \|w^{k + 1} - w^k\|^2_{\sigma^{-1}} + \frac{1}{2} \|w^k - w^{k - 1}\|^2_{\sigma^{-1}}
\end{align}
where we have used the definition of $\kappa$ at the first inequality, Cauchy-Schwarz inequality at the second inequality, Lipschitz continuity of $\nabla_w \phi$ at the third inequality, and Young's inequality at the last inequality.

It follows that
\begin{align*}
&|\inner{w^k - w^{k + 1}}{ \nabla_w \phi(w^k,\delta^k) - q^k}| \\
&\leq n | \inner{w^k - w^{k + 1}}{\nabla_w \phi(w^k,\delta^k) - \nabla_x \phi(w^k,\delta^{k - 1})} | \\
&\quad + | \inner{w^k - w^{k + 1}}{\nabla_w \phi(w^k,\delta^{k - 1}) - \nabla_w \phi(w^{k - 1},\delta^{k - 1})} | \\
&\leq \kappa \|w^{k + 1} - w^k\|^2_{\sigma^{-1}} + \frac{\kappa}{2} \|w^k - w^{k - 1}\|^2_{\sigma^{-1}} + \frac{n\kappa}{2} \|\delta^k - \delta^{k - 1}\|^2_{\tau^{-1}}
\end{align*}
where we have applied the triangle inequality to (\ref{eq:lemma4_1}) at the first inequality and have used (\ref{eq:lemma4_2}) and (\ref{eq:lemma4_3}) at the second inequality.
\end{proof}

\newpage

\begin{lemma}[One-Iteration Result] \label{lemma:one-iter}
Assume $\sigma > 0$, $\tau > 0$, $\theta > 0$. We then have for any $(w,\delta)$,
\begin{align*}
0 &\geq
\begin{bmatrix}
\gamma^{k + 1}_f  + \nabla_w \phi(w^{k + 1},\hat{\delta}^{k + 1}) \vspace{1mm} \\
\hat{\gamma}^{k + 1}_g - \nabla_\delta \phi(w^{k + 1},\hat{\delta}^{k + 1})
\end{bmatrix}^\top
\begin{bmatrix}
w^{k + 1} - w \vspace{1mm} \\
\hat{\delta}^{k + 1} - \delta
\end{bmatrix} \\
&\quad + \E_k \inner{w- w^{k + 1}}{\nabla_w \phi(w^{k + 1},\delta^{k + 1}) - q^{k + 1}} -\theta \inner{w - w^k}{\nabla_w \phi(w^k,\delta^k) - q^k} \\
&\quad + \frac{1}{2} \|w - w^{k + 1}\|^2_{\sigma^{-1}} + \frac{(1 - 2\kappa)}{2} \|w^{k + 1} - w^k\|^2_{\sigma^{-1}} - \frac{1}{2} \|w - w^k\|^2_{\sigma^{-1}} - \frac{\kappa \theta}{2} \|w^k - w^{k - 1}\|^2_{\sigma^{-1}} \\
&\quad + \frac{n}{2} \E_k \|\delta - \delta^{k + 1}\|^2_{\tau^{-1}} + \frac{n}{2} \E_k \|\delta^{k + 1} - \delta^k\|^2_{\tau^{-1}} - \frac{n}{2} \|\delta - \delta^k\|^2_{\tau^{-1}} - \frac{n \kappa \theta}{2} \|\delta^k - \delta^{k - 1}\|^2_{\tau^{-1}}.
\end{align*}
\end{lemma}
\begin{proof}
Optimality condition (\ref{eq:opt1}) implies that
\begin{align}
0 &= \inner{w^{k + 1} - w}{\gamma^{k + 1}_f + \sigma^{-1}[w^{k + 1} - w^k + \sigma\{ \nabla_w \phi(w^k,\delta^k) + \theta (\nabla_w \phi(w^k,\delta^k) - q^k) \}]} \nonumber \\
&= \inner{w^{k + 1} - w}{\gamma^{k + 1}_f + \nabla_w \phi(w^{k + 1},\hat{\delta}^{k + 1})} - \inner{w - w^{k + 1}}{\nabla_w \phi(w^k,\delta^k) - \nabla_w \phi(w^{k + 1},\hat{\delta}^{k + 1})} \nonumber \\
&\quad - \theta \inner{w - w^k}{\nabla_w \phi(w^k,\delta^k) - q^k} - \theta \inner{w^k - w^{k + 1}}{\nabla_w \phi(w^k,\delta^k) - q^k} \nonumber \\
&\quad + \sigma^{-1} \inner{w^{k + 1} - w}{w^{k + 1} - w^k} \nonumber \\
&= \inner{w^{k + 1} - w}{\gamma^{k + 1}_f + \nabla_w \phi(w^{k + 1},\hat{\delta}^{k + 1})} - \inner{w - w^{k + 1}}{\nabla_w \phi(w^k,\delta^k) - \nabla_w \phi(w^{k + 1},\hat{\delta}^{k + 1})} \nonumber \\
&\quad - \theta \inner{w - w^k}{\nabla_w \phi(w^k,\delta^k) - q^k} - \theta \inner{w^k - w^{k + 1}}{\nabla_w \phi(w^k,\delta^k) - q^k} \nonumber \\
&\quad + \frac{1}{2} \|w - w^{k + 1}\|^2_{\sigma^{-1}} + \frac{1}{2} \|w^{k + 1} - w^k\|^2_{\sigma^{-1}} - \frac{1}{2} \|w - w^k\|^2_{\sigma^{-1}}. \label{eq:lemma5_1}
\end{align}
By (\ref{eq:e1}) with $r(\delta) = \nabla_w \phi(w^{k + 1},\delta)$, we have
\begin{align}
&\nabla_w \phi(w^k,\delta^k) - \nabla_w \phi(w^{k + 1},\hat{\delta}^{k + 1}) \nonumber \\
&= \nabla_w \phi(w^k,\delta^k) - (n - 1)\{\E_k \nabla_w \phi(w^{k + 1},\delta^{k + 1}) - \nabla_w \phi(w^{k + 1},\delta^k)\} - \E_k \nabla_w \phi(w^{k + 1},\delta^{k + 1}) \nonumber \\
&= \E_k q^{k + 1} - \E_k \nabla_w \phi(w^{k + 1},\delta^{k + 1}) \label{eq:lemma5_2}
\end{align}
and by plugging this into (\ref{eq:lemma5_1}) and using linearity of expectation, we obtain
\begin{align}
0 &= \inner{w^{k + 1} - w}{\gamma^{k + 1}_f + \nabla_w \phi(w^{k + 1},\hat{\delta}^{k + 1})} + \E_k \inner{w - w^{k + 1}}{\nabla_w \phi(w^{k + 1},\delta^{k + 1}) - q^{k + 1}} \nonumber \\
&\quad - \theta \inner{w - w^k}{\nabla_w \phi(w^k,\delta^k) - q^k} - \theta \inner{w^k - w^{k + 1}}{\nabla_w \phi(w^k,\delta^k) - q^k} \nonumber \\
&\quad + \frac{1}{2} \|w - w^{k + 1}\|^2_{\sigma^{-1}} + \frac{1}{2} \|w^{k + 1} - w^k\|^2_{\sigma^{-1}} - \frac{1}{2} \|w - w^k\|^2_{\sigma^{-1}}. \label{eq:lemma5_3}
\end{align}
Applying (\ref{eq:bound2}) of Lemma \ref{lemma:4} to (\ref{eq:lemma5_3}), we have
\begin{align}
0 &\geq \inner{w^{k + 1} - w}{\gamma^{k + 1}_f + \nabla_w \phi(w^{k + 1},\hat{\delta}^{k + 1})} + \E_k \inner{w - w^{k + 1}}{\nabla_w \phi(w^{k + 1},\delta^{k + 1}) - q^{k + 1}} \nonumber \\
&\quad - \theta \inner{w - w^k}{\nabla_w \phi(w^k,\delta^k) - q^k} + \frac{1}{2} \|w - w^{k + 1}\|^2_{\sigma^{-1}} + \frac{(1 - 2\kappa)}{2} \|w^{k + 1} - w^k\|^2_{\sigma^{-1}} - \frac{1}{2} \|w - w^k\|^2_{\sigma^{-1}} \nonumber \\
&\quad  - \frac{\kappa \theta}{2} \|w^k - w^{k - 1}\|^2_{\sigma^{-1}} - \frac{n \kappa \theta}{2} \|\delta^k - \delta^{k - 1}\|^2_{\tau^{-1}}. \label{eq:lemma5_4}
\end{align}
Optimality condition (\ref{eq:opt3}) implies that
\begin{align}
0 &= \inner{\hat{\delta}^{k + 1} - \delta}{\hat{\gamma}^{k + 1}_g + \tau^{-1}(\hat{\delta}^{k + 1} - \delta^k - \tau \nabla_\delta \phi(w^{k + 1},\hat{\delta}^{k + 1}))} \nonumber \\
&= \inner{\hat{\delta}^{k + 1} - \delta}{\hat{\gamma}^{k + 1}_g -\nabla_\delta \phi(w^{k + 1},\hat{\delta}^{k + 1})} + \tau^{-1} \inner{\hat{\delta}^{k + 1} - \delta}{\hat{\delta}^{k + 1} - \delta^k} \nonumber \\
&= \inner{\hat{\delta}^{k + 1} - \delta}{\hat{\gamma}^{k + 1}_g -\nabla_\delta \phi(w^{k + 1},\hat{\delta}^{k + 1})} + \frac{1}{2} \|\delta - \hat{\delta}^{k + 1}\|^2_{\tau^{-1}} + \frac{1}{2} \|\hat{\delta}^{k + 1} - \delta^k\|^2_{\tau^{-1}} - \frac{1}{2} \|\delta - \delta^k\|^2_{\tau^{-1}} \nonumber \\
&= \inner{\hat{\delta}^{k + 1} - \delta}{\hat{\gamma}^{k + 1}_g -\nabla_\delta \phi(w^{k + 1},\hat{\delta}^{k + 1})} \nonumber \\
&\quad + \frac{n}{2} \E_k \|\delta - \delta^{k + 1}\|^2_{\tau^{-1}} + \frac{n}{2} \E_k \|\delta^{k + 1} - \delta^k\|^2_{\tau^{-1}} - \frac{n}{2} \|\delta - \delta^k\|^2_{\tau^{-1}} \label{eq:lemma5_5}
\end{align}
where we have used (\ref{eq:e2}) and (\ref{eq:e3}) at the last equality. Adding (\ref{eq:lemma5_4}) and (\ref{eq:lemma5_5}) concludes the proof.
\end{proof}

\newpage

\subsection{Proof of Theorem \ref{thm:SSI-HG1}} \label{append:thm1_proof}

\begin{lemma} \label{lemma:6}
Let $\theta = 1$. We then have
\begin{align*}
\left\|
\begin{bmatrix}
\gamma^{k + 1}_f + \nabla_w \phi(w^{k + 1},\hat{\delta}^{k + 1}) \vspace{1mm} \\
\hat{\gamma}^{k + 1}_g - \nabla_\delta \phi(x^{k + 1},\hat{\delta}^{k + 1})
\end{bmatrix}
\right\|^2 &\leq 3(\sigma^{-1} + 2L_{11}^2\sigma) \|w^{k + 1} - w^k\|^2_{\sigma^{-1}} + 6L_{11}^2\sigma \|w^k - w^{k - 1}\|^2_{\sigma^{-1}} \\
&\quad + n(\tau^{-1} + 6n L_{12}^2 \tau) \E_k \|\delta^{k + 1} - \delta^k\|^2_{\tau^{-1}} + 6n^2 L_{12}^2 \tau \|\delta^k - \delta^{k - 1}\|^2_{\tau^{-1}}.
\end{align*}
\end{lemma}
\begin{proof}
First, observe that
\begin{align}
&\| \gamma^{k + 1}_f + \nabla_w \phi(w^{k + 1},\hat{\delta}^{k + 1}) \|^2 \nonumber \\
&= \| \sigma^{-1}(- w^{k + 1} + w^k) + \E_k \{\nabla_w \phi(w^{k + 1},\delta^{k + 1}) - q^{k + 1}\} - \{ \nabla_w \phi(w^k,\delta^k) - q^k \} \|^2 \nonumber \\
&\leq 3\sigma^{-1} \|w^{k + 1} - w^k\|^2_{\sigma^{-1}} + 3 \| \E_k \{\nabla_w \phi(w^{k + 1},\delta^{k + 1}) - q^{k + 1}\} \|^2 + 3 \| \nabla_w \phi(w^k,\delta^k) - q^k \|^2 \nonumber \\
&\leq 3\sigma^{-1} \|w^{k + 1} - w^k\|^2_{\sigma^{-1}} + 3 \E_k \| \nabla_w \phi(w^{k + 1},\delta^{k + 1}) - q^{k + 1} \|^2 + 3 \| \nabla_w \phi(w^k,\delta^k) - q^k \|^2 \label{eq:lemma6_1}
\end{align}
where we have used the combination of (\ref{eq:opt1}) and (\ref{eq:lemma5_2}) at the first equality and Jensen's inequality at the first and second inequalities. We have
\begin{align}
&\| \nabla_w \phi(w^k,\delta^k) - q^k \|^2 \nonumber \\
&= \| \nabla_w \phi(w^k,\delta^k) + (n - 1)\{ \nabla_w \phi(w^k,\delta^k) - \nabla_w \phi(w^k,\delta^{k - 1}) \} - \nabla_w \phi(w^{k - 1},\delta^{k - 1}) \|^2 \nonumber \\
&= \| n \{ \nabla_w \phi(w^k,\delta^k) - \nabla_w \phi(w^k,\delta^{k - 1}) \} + \{ \nabla_w \phi(w^k,\delta^{k - 1}) - \nabla_w \phi(w^{k - 1},\delta^{k - 1}) \} \|^2 \nonumber \\
&\leq 2n^2 \| \nabla_w \phi(w^k,\delta^k) - \nabla_w \phi(w^k,\delta^{k - 1}) \|^2 + 2 \| \nabla_w \phi(w^k,\delta^{k - 1}) - \nabla_w \phi(w^{k - 1},\delta^{k - 1}) \|^2 \nonumber \\
&\leq 2n^2 L_{12}^2 \|\delta^k - \delta^{k - 1}\|^2 + 2L_{11}^2 \|w^k - w^{k - 1}\|^2 \nonumber \\
&= 2n^2 L_{12}^2 \tau \|\delta^k - \delta^{k - 1}\|^2_{\tau^{-1}} + 2L_{11}^2 \sigma \|w^k - w^{k - 1}\|^2_{\sigma^{-1}} \label{eq:lemma6_2}
\end{align}
where we have used the definition of $q^k$ at the first equality, Jensen's inequality at the first inequality, and Lipschitz continuity of $\nabla_w \phi(w,\delta)$ at the second inequality. Similarly,
\begin{align}
\| \nabla_w \phi(w^{k + 1},\delta^{k + 1}) - q^{k + 1} \|^2 \leq 2n^2 L_{12}^2 \tau \|\delta^{k + 1} - \delta^k\|^2_{\tau^{-1}} +  2L_{11}^2 \sigma \|w^{k + 1} - w^k\|^2_{\sigma^{-1}}. \label{eq:lemma6_3}
\end{align}
Applying (\ref{eq:lemma6_2}) and (\ref{eq:lemma6_3}) to (\ref{eq:lemma6_1}), we obtain
\begin{align}
\| \gamma^{k + 1}_f + \nabla_w \phi(w^{k + 1},\hat{\delta}^{k + 1}) \|^2 &\leq 3(\sigma^{-1} + 2L_{11}^2\sigma) \|w^{k + 1} - w^k\|^2_{\sigma^{-1}} + 6L_{11}^2\sigma \|w^k - w^{k - 1}\|^2_{\sigma^{-1}} \nonumber \\
&\quad + 6n^2 L_{12}^2 \tau \E_k \|\delta^{k + 1} - \delta^k\|^2_{\tau^{-1}} + 6n^2 L_{12}^2 \tau \|\delta^k - \delta^{k - 1}\|^2_{\tau^{-1}}. \label{eq:lemma6_4}
\end{align}
We also have
\begin{align}
\| \hat{\gamma}^{k + 1}_g - \nabla_\delta \phi(w^{k + 1},\hat{\delta}^{k + 1}) \|^2 &= \| \tau^{-1}(-\hat{\delta}^{k + 1} + \delta^k) \|^2 \nonumber \\
&= \tau^{-1} \|\hat{\delta}^{k + 1} - \delta^k\|^2_{\tau^{-1}} \nonumber \\
&= n \tau^{-1} \E_k \|\delta^{k + 1} - \delta^k\|^2_{\tau^{-1}} \label{eq:lemma6_5}
\end{align}
where we have used (\ref{eq:opt3}) at the first equality and (\ref{eq:e3}) at the third equality. Combining (\ref{eq:lemma6_4}) and (\ref{eq:lemma6_5}) concludes the proof.
\end{proof}

\newpage

\textbf{Theorem 1.} \textit{Suppose Assumptions \ref{assump:1} and \ref{assump:2} are true. Let $\{(w^k,\delta^k)\}$ be the sequence generated by SSI-HG, and define the full-dimensional update (which only depends on $w^k$ and $\delta^{k - 1}$)}
\begin{align*}
\hat{\delta}^k = \prox^\tau_g [\delta^{k - 1} + \tau \nabla_\delta \phi(w^k,\hat{\delta}^k)].
\end{align*}
\textit{If $\theta = 1$ and $\rho > 0$, $\sigma > 0$, $\tau > 0$ satisfy}
\begin{align*}
\max\{L_{12} (\sigma \tau n)^{1/2}, L_{11} \sigma\} < \min\{1/3 - \rho(\sigma^{-1} + 4L_{11}^2\sigma), 1 - \rho(\tau^{-1} + 12 n L_{12}^2 \tau)\},
\end{align*}
\textit{we have}
\begin{align*}
\frac{1}{K} \sum_{k = 1}^K \E \|F(w^k,\hat{\delta}^k)\|^2 = O(1/K).
\end{align*}
\begin{proof}
Let $(w^*,\delta^*)$ be a solution of the weak MVI problem (which exists by Assumption \ref{assump:2}). Then
\begin{align}
0 &\geq
\begin{bmatrix}
\gamma^{k + 1}_f  + \nabla_w \phi(w^{k + 1},\hat{\delta}^{k + 1}) \vspace{1mm} \\
\hat{\gamma}^{k + 1}_g - \nabla_\delta \phi(w^{k + 1},\hat{\delta}^{k + 1})
\end{bmatrix}^\top
\begin{bmatrix}
w^{k + 1} - w^* \vspace{1mm} \\
\hat{\delta}^{k + 1} - \delta^*
\end{bmatrix} \nonumber \\
&\quad + \E_k \inner{w^* - w^{k + 1}}{\nabla_w \phi(w^{k + 1},\delta^{k + 1}) - q^{k + 1}} - \inner{w^* - w^k}{\nabla_w \phi(w^k,\delta^k) - q^k} \nonumber \\
&\quad + \frac{1}{2} \|w^* - w^{k + 1}\|^2_{\sigma^{-1}} + \frac{(1 - 2\kappa)}{2} \|w^{k + 1} - w^k\|^2_{\sigma^{-1}} - \frac{1}{2} \|w^* - w^k\|^2_{\sigma^{-1}} - \frac{\kappa}{2} \|w^k - w^{k - 1}\|^2_{\sigma^{-1}} \nonumber \\
&\quad + \frac{n}{2} \E_k \|\delta^* - \delta^{k + 1}\|^2_{\tau^{-1}} + \frac{n}{2} \E_k \|\delta^{k + 1} - \delta^k\|^2_{\tau^{-1}} - \frac{n}{2} \|\delta^* - \delta^k\|^2_{\tau^{-1}} - \frac{n \kappa}{2} \|\delta^k - \delta^{k - 1}\|^2_{\tau^{-1}} \nonumber \\
&\geq -\frac{\rho}{2} \left\|
\begin{bmatrix}
\gamma^{k + 1}_f + \nabla_w \phi(w^{k + 1},\hat{\delta}^{k + 1}) \vspace{1mm} \\
\hat{\gamma}^{k + 1}_g - \nabla_y \phi(w^{k + 1},\hat{\delta}^{k + 1})
\end{bmatrix}
\right\|^2 \nonumber \\
&\quad + \E_k \inner{w^* - w^{k + 1}}{\nabla_w \phi(w^{k + 1},\delta^{k + 1}) - q^{k + 1}} - \inner{w^* - w^k}{\nabla_w \phi(w^k,\delta^k) - q^k} \nonumber \\
&\quad + \frac{1}{2} \|w^* - w^{k + 1}\|^2_{\sigma^{-1}} + \frac{(1 - 2\kappa)}{2} \|w^{k + 1} - w^k\|^2_{\sigma^{-1}} - \frac{1}{2} \|w^*- w^k\|^2_{\sigma^{-1}} - \frac{\kappa}{2} \|w^k - w^{k - 1}\|^2_{\sigma^{-1}} \nonumber \\
&\quad + \frac{n}{2} \E_k \|\delta^* - \delta^{k + 1}\|^2_{\tau^{-1}} + \frac{n}{2} \E_k \|\delta^{k + 1} - \delta^k\|^2_{\tau^{-1}} - \frac{n}{2} \|\delta^* - \delta^k\|^2_{\tau^{-1}} - \frac{n \kappa}{2} \|\delta^k - \delta^{k - 1}\|^2_{\tau^{-1}} \nonumber \\
&\geq \E_k \inner{w^* - w^{k + 1}}{\nabla_w \phi(w^{k + 1},\delta^{k + 1}) - q^{k + 1}} - \inner{w^* - w^k}{\nabla_w \phi(w^k,\delta^k) - q^k} \nonumber \\
&\quad + \frac{1}{2} \|w^* - w^{k + 1}\|^2_{\sigma^{-1}} + \frac{\{1 - 2\kappa - 3 \rho (\sigma^{-1} + 2L_{11}^2\sigma)\}}{2} \|w^{k + 1} - w^k\|^2_{\sigma^{-1}} \nonumber \\
&\quad - \frac{1}{2} \|w^*- w^k\|^2_{\sigma^{-1}} - \frac{\{\kappa + \rho (6 L_{11}^2\sigma)\}}{2} \|w^k - w^{k - 1}\|^2_{\sigma^{-1}} \nonumber \\
&\quad + \frac{n}{2} \E_k \|\delta^* - \delta^{k + 1}\|^2_{\tau^{-1}} + \frac{n\{ 1 - \rho (\tau^{-1} + 6 n L_{12}^2 \tau) \}}{2} \E_k \|\delta^{k + 1} - \delta^k\|^2_{\tau^{-1}} \nonumber \\
&\quad - \frac{n}{2} \|\delta^* - \delta^k\|^2_{\tau^{-1}} - \frac{n \{\kappa + \rho (6 n L_{12}^2 \tau)\}}{2} \|\delta^k - \delta^{k - 1}\|^2_{\tau^{-1}} \label{eq:thm1_1}
\end{align}
where we have used Lemma \ref{lemma:one-iter} with $\theta = 1$ at the first inequality, Assumption \ref{assump:2} at the second inequality, and Lemma \ref{lemma:6} at the third inequality. Rearranging and adding and subtracting some terms in (\ref{eq:thm1_1}), we obtain
\begin{align}
&\inner{w^* - w^k}{\nabla_w \phi(w^k,\delta^k) - q^k} + \frac{1}{2} \|w^*- w^k\|^2_{\sigma^{-1}} + \frac{n}{2} \|\delta^* - \delta^k\|^2_{\tau^{-1}} \nonumber \\
&\quad + \frac{\{\kappa + \rho (6 L_{11}^2\sigma)\}}{2} \|w^k - w^{k - 1}\|^2_{\sigma^{-1}} + \frac{n \{\kappa + \rho (6 n L_{12}^2 \tau)\}}{2} \|\delta^k - \delta^{k - 1}\|^2_{\tau^{-1}} \nonumber \\
&\geq \E_k \inner{w^* - w^{k + 1}}{\nabla_w \phi(w^{k + 1},\delta^{k + 1}) - q^{k + 1}} + \frac{1}{2} \|w^* - w^{k + 1}\|^2_{\sigma^{-1}} + \frac{n}{2} \E_k \|\delta^* - \delta^{k + 1}\|^2_{\tau^{-1}} \nonumber \\
&\quad + \frac{\{\kappa + \rho (6 L_{11}^2\sigma)\}}{2} \|w^{k + 1} - w^k\|^2_{\sigma^{-1}} + \frac{n \{\kappa + \rho (6 n L_{12}^2 \tau)\}}{2} \|\delta^{k + 1} - \delta^k\|^2_{\tau^{-1}} \nonumber \\
&\quad + \frac{\{1 - 3\kappa - 3 \rho(\sigma^{-1} + 4L_{11}^2\sigma)\}}{2} \|w^{k + 1} - w^k\|^2_{\sigma^{-1}} \nonumber \\
&\quad + \frac{n\{ 1 - \kappa - \rho (\tau^{-1} + 12 n L_{12}^2 \tau) \}}{2} \E_k \|\delta^{k + 1} - \delta^k\|^2_{\tau^{-1}}. \label{eq:thm1_2}
\end{align}
Taking full expectation over (\ref{eq:thm1_2}), summing both sides over $k = 0, \ldots, K - 1$, and using $(w^{-1},\delta^{-1}) = (w^0,\delta^0)$, we obtain
\begin{align}
& \frac{1}{2} \|w^* - w^0\|^2_{\sigma^{-1}} + \frac{n}{2} \|\delta^* - \delta^0\|^2_{\tau^{-1}} \nonumber \\
&\geq \E \inner{w^* - w^K}{\nabla_w \phi(w^K,\delta^K) - q^K} + \frac{1}{2} \|w^* - w^K\|^2_{\sigma^{-1}} + \frac{n}{2} \E \|\delta^* - \delta^K\|^2_{\tau^{-1}} \nonumber \\
&\quad + \frac{\{\kappa + \rho (6 L_{11}^2\sigma)\}}{2} \E \|w^K - w^{K - 1}\|^2_{\sigma^{-1}} + \frac{n \{\kappa + \rho (6 n L_{12}^2 \tau)\}}{2} \E \|\delta^K - \delta^{K - 1}\|^2_{\tau^{-1}} \nonumber \\
&\quad + \frac{\{1 - 3\kappa - 3 \rho (\sigma^{-1} + 4L_{11}^2\sigma)\}}{2} \sum_{k = 0}^{K - 1} \E \|w^{k + 1} - w^k\|^2_{\sigma^{-1}} \nonumber \\
&\quad + \frac{n\{ 1 - \kappa - \rho (\tau^{-1} + 12 n L_{12}^2 \tau) \}}{2} \sum_{k = 0}^{K - 1} \E \|\delta^{k + 1} - \delta^k\|^2_{\tau^{-1}}. \label{eq:thm1_3}
\end{align}
By reasoning similarly as the proof of Lemma \ref{lemma:4}, we have
\begin{align}
| \inner{w^* - w^K}{\nabla_w \phi(w^K,\delta^K) - q^K} | \leq \kappa \|w^* - w^K\|^2_{\sigma^{-1}} + \frac{\kappa}{2} \|w^K - w^{K - 1}\|^2_{\sigma^{-1}} + \frac{n\kappa}{2} \|\delta^K - \delta^{K - 1}\|^2_{\tau^{-1}} \label{eq:inner_bound}
\end{align}
and so (\ref{eq:thm1_3}) implies
\begin{align}
&\frac{1}{2} \|w^* - w^0\|^2_{\sigma^{-1}} + \frac{n}{2} \|\delta^* - \delta^0\|^2_{\tau^{-1}} \nonumber \\
&\geq \frac{(1 - 2\kappa)}{2} \|w^* - w^K\|^2_{\sigma^{-1}} + \frac{n}{2} \E \|\delta^* - \delta^K\|^2_{\tau^{-1}} \nonumber \\
&\quad + \frac{\rho (6 L_{11}^2\sigma)}{2} \E \|w^K - w^{K - 1}\|^2_{\sigma^{-1}} + \frac{\rho (6 n^2 L_{12}^2 \tau)}{2} \E \|\delta^K - \delta^{K - 1}\|^2_{\tau^{-1}} \nonumber \\
&\quad + \frac{\{1 - 3\kappa - 3 \rho (\sigma^{-1} + 4L_{11}^2\sigma) \}}{2} \sum_{k = 0}^{K - 1} \E \|w^{k + 1} - w^k\|^2_{\sigma^{-1}} \nonumber \\
&\quad + \frac{n\{ 1 - \kappa - \rho (\tau^{-1} + 12 n L_{12}^2 \tau) \}}{2} \sum_{k = 0}^{K - 1} \E \|\delta^{k + 1} - \delta^k\|^2_{\tau^{-1}} \label{eq:thm1_4}
\end{align}
All the coefficients for the quadratic terms in (\ref{eq:thm1_4}) are positive by the definition of $\kappa$ (in Lemma \ref{lemma:4}) and the step-size conditions. Hence, we may remove the first four quadratic terms at the RHS of (\ref{eq:thm1_4}) to obtain
\begin{align}
&\frac{1}{2} \|w^* - w^0\|^2_{\sigma^{-1}} + \frac{n}{2} \|\delta^* - \delta^0\|^2_{\tau^{-1}} \nonumber \\
&\geq \frac{\{1 - 3\kappa - 3 \rho (\sigma^{-1} + 4L_{11}^2\sigma)\}}{2} \sum_{k = 0}^{K - 1} \E \|w^{k + 1} - w^k\|^2_{\sigma^{-1}} \nonumber \\
&\quad + \frac{n\{ 1 - \kappa - \rho (\tau^{-1} + 12 n L_{12}^2 \tau) \}}{2} \sum_{k = 0}^{K - 1} \E \|\delta^{k + 1} - \delta^k\|^2_{\tau^{-1}}. \label{eq:thm1_5}
\end{align}
This proves that
\begin{align}
\frac{1}{K} \sum_{k = 0}^{K - 1} \E \|w^{k + 1} - w^k\|^2 = O(1/K), \qquad \frac{1}{K} \sum_{k = 0}^{K - 1} \E \|\delta^{k + 1} - \delta^k\|^2 = O(1/K). \label{eq:thm1_6}
\end{align}
By Lemma \ref{lemma:6} and the definition of the saddle subdifferential norm,
\begin{align}
\E \|F(w^{k + 1},\hat{\delta}^{k + 1})\|^2 &\leq 3(\sigma^{-1} + 2L_{11}^2\sigma) \E \|w^{k + 1} - w^k\|^2_{\sigma^{-1}} + 6L_{11}^2\sigma \E \|w^k - w^{k - 1}\|^2_{\sigma^{-1}} \nonumber \\
&\quad + n(\tau^{-1} + 6n L_{12}^2 \tau) \E \|\delta^{k + 1} - \delta^k\|^2_{\tau^{-1}} + 6n^2 L_{12}^2 \tau \E \|\delta^k - \delta^{k - 1}\|^2_{\tau^{-1}}. \label{eq:thm1_7}
\end{align}
Averaging (\ref{eq:thm1_7}) over $k = 0, \ldots, K - 1$ and using (\ref{eq:thm1_6}) concludes the proof.
\end{proof}

\newpage

\subsection{Proof of Theorem \ref{thm:SSI-HG2}} \label{append:thm2_proof}

\textbf{Theorem 2.} \textit{Suppose Assumptions \ref{assump:1} and \ref{assump:3} are true. Let $\{(w^k,\delta^k)\}$ be the sequence generated by SSI-HG. If $\theta > 0$, $\sigma > 0$, $\tau > 0$ satisfy}
\begin{align*}
\max\{L_{12} (\sigma \tau n)^{1/2}, L_{11} \sigma\} \leq 1/3, \qquad \theta = \max\left\{ \frac{1}{1 + \mu \sigma}, \frac{1 + (n - 1) \mu \tau / n}{1 + \mu \tau} \right\},
\end{align*}
\textit{we have}
\begin{align*}
\E \|w^* - w^K\|^2 = O(\theta^K), \qquad \E \|\delta^* - \delta^K\|^2 = O(\theta^K).
\end{align*}
\begin{proof}
Let $(w^*,\delta^*)$ be a solution of the strong MVI problem (which exists by Assumption \ref{assump:3}). Then
\begin{align}
0 &\geq
\begin{bmatrix}
\gamma^{k + 1}_f  + \nabla_w \phi(w^{k + 1},\hat{\delta}^{k + 1}) \vspace{1mm} \\
\hat{\gamma}^{k + 1}_g - \nabla_y \phi(w^{k + 1},\hat{\delta}^{k + 1})
\end{bmatrix}^\top
\begin{bmatrix}
w^{k + 1} - w^* \vspace{1mm} \\
\hat{\delta}^{k + 1} - \delta^*
\end{bmatrix} \nonumber \\
&\quad + \E_k \inner{w^* - w^{k + 1}}{\nabla_w \phi(w^{k + 1},\delta^{k + 1}) - q^{k + 1}} - \theta \inner{w^* - w^k}{\nabla_w \phi(w^k,\delta^k) - q^k} \nonumber \\
&\quad + \frac{1}{2} \|w^* - w^{k + 1}\|^2_{\sigma^{-1}} + \frac{(1 - 2\kappa)}{2} \|w^{k + 1} - w^k\|^2_{\sigma^{-1}} - \frac{1}{2} \|w^* - w^k\|^2_{\sigma^{-1}} - \frac{\kappa \theta}{2} \|w^k - w^{k - 1}\|^2_{\sigma^{-1}} \nonumber \\
&\quad + \frac{n}{2} \E_k \|\delta^* - \delta^{k + 1}\|^2_{\tau^{-1}} + \frac{n}{2} \E_k \|\delta^{k + 1} - \delta^k\|^2_{\tau^{-1}} - \frac{n}{2} \|\delta^* - \delta^k\|^2_{\tau^{-1}} - \frac{n \kappa \theta}{2} \|\delta^k - \delta^{k - 1}\|^2_{\tau^{-1}} \nonumber \\
&\geq \frac{\mu}{2} \left\|
\begin{bmatrix}
w^{k + 1} - w^*\\
\hat{\delta}^{k + 1} - \delta^*
\end{bmatrix}
\right\|^2 \nonumber \\
&\quad + \E_k \inner{w^* - w^{k + 1}}{\nabla_w \phi(w^{k + 1},\delta^{k + 1}) - q^{k + 1}} - \theta \inner{w^* - w^k}{\nabla_w \phi(w^k,\delta^k) - q^k} \nonumber \\
&\quad + \frac{1}{2} \|w^* - w^{k + 1}\|^2_{\sigma^{-1}} + \frac{(1 - 2\kappa)}{2} \|w^{k + 1} - w^k\|^2_{\sigma^{-1}} - \frac{1}{2} \|w^*- w^k\|^2_{\sigma^{-1}} - \frac{\kappa \theta}{2} \|w^k - w^{k - 1}\|^2_{\sigma^{-1}} \nonumber \\
&\quad + \frac{n}{2} \E_k \|\delta^* - \delta^{k + 1}\|^2_{\tau^{-1}} + \frac{n}{2} \E_k \|\delta^{k + 1} - \delta^k\|^2_{\tau^{-1}} - \frac{n}{2} \|\delta^* - \delta^k\|^2_{\tau^{-1}} - \frac{n \kappa \theta}{2} \|\delta^k - \delta^{k - 1}\|^2_{\tau^{-1}} \nonumber \\
&= \E_k \inner{w^* - w^{k + 1}}{\nabla_w \phi(w^{k + 1},\delta^{k + 1}) - q^{k + 1}} - \theta \inner{w^* - w^k}{\nabla_w \phi(w^k,\delta^k) - q^k} \nonumber \\
&\quad + \frac{(1 + \mu \sigma)}{2} \|w^* - w^{k + 1}\|^2_{\sigma^{-1}} + \frac{(1 - 2\kappa)}{2} \|w^{k + 1} - w^k\|^2_{\sigma^{-1}} - \frac{1}{2} \|w^*- w^k\|^2_{\sigma^{-1}} - \frac{\kappa \theta}{2} \|w^k - w^{k - 1}\|^2_{\sigma^{-1}} \nonumber \\
&\quad + \frac{n(1 + \mu \tau)}{2} \E_k \|\delta^* - \delta^{k + 1}\|^2_{\tau^{-1}} + \frac{n}{2} \E_k \|\delta^{k + 1} - \delta^k\|^2_{\tau^{-1}} \nonumber \\
&\quad - \frac{n\{1 + (n - 1) \mu \tau / n\}}{2} \|\delta^* - \delta^k\|^2_{\tau^{-1}} - \frac{n \kappa \theta}{2} \|\delta^k - \delta^{k - 1}\|^2_{\tau^{-1}} \label{eq:thm2_1}
\end{align}
where we have used Lemma \ref{lemma:one-iter} at the first inequality, Assumption \ref{assump:3} at the second inequality, and (\ref{eq:e2}) at the first equality. Rearranging and adding and subtracting some terms in (\ref{eq:thm2_1}), we obtain
\begin{align}
&\theta \inner{w^* - w^k}{\nabla_w \phi(w^k,\delta^k) - q^k} + \frac{1}{2} \|w^*- w^k\|^2_{\sigma^{-1}} + \frac{n\{1 + (n - 1) \mu \tau / n\}}{2} \|\delta^* - \delta^k\|^2_{\tau^{-1}} \nonumber \\
&\quad + \frac{\kappa \theta}{2} \|w^k - w^{k - 1}\|^2_{\sigma^{-1}} + \frac{n \kappa \theta}{2} \|\delta^k - \delta^{k - 1}\|^2_{\tau^{-1}} \nonumber \\
&\geq \E_k \inner{w^* - w^{k + 1}}{\nabla_w \phi(w^{k + 1},\delta^{k + 1}) - q^{k + 1}} + \frac{(1 + \mu \sigma)}{2} \|w^* - w^{k + 1}\|^2_{\sigma^{-1}} \nonumber \\
&\quad + \frac{n(1 + \mu \tau)}{2} \E_k \|\delta^* - \delta^{k + 1}\|^2_{\tau^{-1}} + \frac{\kappa}{2} \|w^{k + 1} - w^k\|^2_{\sigma^{-1}} + \frac{n \kappa}{2} \|\delta^{k + 1} - \delta^k\|^2_{\tau^{-1}} \nonumber \\
&\quad + \frac{(1 - 3\kappa)}{2} \|w^{k + 1} - w^k\|^2_{\sigma^{-1}} + \frac{n(1 - \kappa)}{2} \E_k \|\delta^{k + 1} - \delta^k\|^2_{\tau^{-1}}. \label{eq:thm2_2}
\end{align}
All the coefficients for the quadratic terms in (\ref{eq:thm2_2}) are non-negative by the definition of $\kappa$ (in Lemma \ref{lemma:4}) and the step-size conditions. Hence, we may remove the last two quadratic terms at the RHS of (\ref{eq:thm2_2}) and take full expectation to obtain
\begin{align}
&\theta \E \inner{w^* - w^k}{\nabla_w \phi(w^k,\delta^k) - q^k} + \frac{1}{2} \E \|w^*- w^k\|^2_{\sigma^{-1}} + \frac{n\{1 + (n - 1) \mu \tau / n\}}{2} \E \|\delta^* - \delta^k\|^2_{\tau^{-1}} \nonumber \\
&\quad + \frac{\kappa \theta}{2} \E \|w^k - w^{k - 1}\|^2_{\sigma^{-1}} + \frac{n \kappa \theta}{2} \E \|\delta^k - \delta^{k - 1}\|^2_{\tau^{-1}} \nonumber \\
&\geq \E \inner{w^* - w^{k + 1}}{\nabla_w \phi(w^{k + 1},\delta^{k + 1}) - q^{k + 1}} + \frac{(1 + \mu \sigma)}{2} \E \|w^* - w^{k + 1}\|^2_{\sigma^{-1}} \nonumber \\
&\quad + \frac{n(1 + \mu \tau)}{2} \E \|\delta^* - \delta^{k + 1}\|^2_{\tau^{-1}} + \frac{\kappa}{2} \E \|w^{k + 1} - w^k\|^2_{\sigma^{-1}} + \frac{n \kappa}{2} \E \|\delta^{k + 1} - \delta^k\|^2_{\tau^{-1}}. \label{eq:thm2_3}
\end{align}
Combining (\ref{eq:thm2_3}) and the definition of $\theta$, we have
\begin{align*}
&\theta \bigg[ \E \inner{w^* - w^k}{\nabla_w \phi(w^k,\delta^k) - q^k} + \frac{(1 + \mu \sigma)}{2} \E \|w^* - w^k\|^2_{\sigma^{-1}} \nonumber \\
&\quad + \frac{n(1 + \mu \tau)}{2} \E \|\delta^* - \delta^k\|^2_{\tau^{-1}} + \frac{\kappa}{2} \E \|w^k - w^{k - 1}\|^2_{\sigma^{-1}} + \frac{n \kappa}{2} \E \|\delta^k - \delta^{k - 1}\|^2_{\tau^{-1}} \bigg] \nonumber \\
&\geq \E \inner{w^* - w^{k + 1}}{\nabla_w \phi(w^{k + 1},\delta^{k + 1}) - q^{k + 1}} + \frac{(1 + \mu \sigma)}{2} \E \|w^* - w^{k + 1}\|^2_{\sigma^{-1}} \nonumber \\
&\quad + \frac{n(1 + \mu \tau)}{2} \E \|\delta^* - \delta^{k + 1}\|^2_{\tau^{-1}} + \frac{\kappa}{2} \E \|w^{k + 1} - w^k\|^2_{\sigma^{-1}} + \frac{n \kappa}{2} \E \|\delta^{k + 1} - \delta^k\|^2_{\tau^{-1}}
\end{align*}
which, together with the fact that $(w^{-1},\delta^{-1}) = (w^0,\delta^0)$, establishes
\begin{align}
&\theta^K \bigg[ \frac{(1 + \mu \sigma)}{2} \|w^* - w^0\|^2_{\sigma^{-1}} + \frac{n(1 + \mu \tau)}{2} \|\delta^* - \delta^0\|^2_{\tau^{-1}} \bigg] \nonumber \\
&\geq \E \inner{w^* - w^K}{\nabla_w \phi(w^K,\delta^K) - q^K} + \frac{(1 + \mu \sigma)}{2} \E \|w^* - w^K\|^2_{\sigma^{-1}} \nonumber \\
&\quad + \frac{n(1 + \mu \tau)}{2} \E \|\delta^* - \delta^K\|^2_{\tau^{-1}} + \frac{\kappa}{2} \E \|w^K - w^{K - 1}\|^2_{\sigma^{-1}} + \frac{n \kappa}{2} \E \|\delta^K - \delta^{K - 1}\|^2_{\tau^{-1}} \nonumber \\
&\geq \frac{(1 + \mu \sigma - 2\kappa)}{2} \E \|w^* - w^K\|^2_{\sigma^{-1}} + \frac{n(1 + \mu \tau)}{2} \E \|\delta^* - \delta^K\|^2_{\tau^{-1}} \label{eq:thm2_4}
\end{align}
where we have used (\ref{eq:inner_bound}) at the last inequality. By the definition of $\kappa$ and the step-size conditions, all the coefficients for the quadratic terms at the RHS of (\ref{eq:thm2_4}) are positive. This concludes the proof.
\end{proof}

\subsection{Corollaries \ref{cor:DSI-HG1} and \ref{cor:DSI-HG2}} \label{append:corollaries}

\textbf{Corollary 3.} \textit{Suppose Assumptions \ref{assump:1} and \ref{assump:2} are true. Let $\{(w^k,\delta^k)\}$ be the sequence generated by DSI-HG. If $\theta = 1$ and $\rho > 0$, $\sigma > 0$, $\tau > 0$ satisfy}
\begin{align*}
\max\{L_{12} (\sigma \tau)^{1/2}, L_{11} \sigma\} < \min\{1/3 - \rho(\sigma^{-1} + 4L_{11}^2\sigma), 1 - \rho(\tau^{-1} + 12 L_{12}^2 \tau)\},
\end{align*}
\textit{we have}
\begin{align*}
\frac{1}{K} \sum_{k = 1}^K \|F(w^k,\delta^k)\|^2 = O(1/K).
\end{align*}

\textbf{Corollary 4.} \textit{Suppose Assumptions \ref{assump:1} and \ref{assump:3} are true. Let $\{(w^k,\delta^k)\}$ be the sequence generated by DSI-HG. If $\theta > 0$, $\sigma > 0$, $\tau > 0$ satisfy}
\begin{align*}
\max\{L_{12} (\sigma \tau)^{1/2}, L_{11} \sigma\} \leq 1/3, \qquad \theta = \max\left\{ \frac{1}{1 + \mu \sigma}, \frac{1}{1 + \mu \tau} \right\},
\end{align*}
\textit{we have}
\begin{align*}
\|w^* - w^K\|^2 = O(\theta^K), \qquad \|\delta^* - \delta^K\|^2 = O(\theta^K).
\end{align*}


\newpage

\section{OMITTED EXPERIMENTS SETTINGS} \label{append:settings}

All images are normalized into the range $[0,1]$.




\textbf{YOPO.} For YOPO-$M$-$N$, we use $M = 10$ and $N = 5$. YOPO-5-10, which is compared with PGD-40 AT on MNIST in the paper for YOPO \citep{zhang2019}, performed worse than YOPO-10-5. Other than the learning rate schedule and $(M,N)$, we use the code and the exact hyperparameter choices released by the authors of YOPO.

\textbf{DAT.} For DAT training, maximum first-order stationary condition (FOSC) value is set to $0.5$, and FOSC control epoch is set to $0.8$ times the number of training epochs. Both are the parameter settings used by the authors of DAT \citep{wang2019}.

\textbf{Other settings.} On SVHN and CIFAR-10, we use random horizontal flipping and random cropping augmentations. For MSI-HG, to promote exploration of the constraint set, we apply augmentation to $\delta$ as well.
All other settings are described in Section \ref{sec:exp}.


\begin{table}[H]
\caption{Hyperparameter choices for MSI-HG in Section \ref{sec:exp}.}
\label{table:MSI_hyper_exp}
\centering
\vspace{1.0em}
\def\arraystretch{1.2}
\begin{tabular}{c c c c}
\toprule
\textbf{Dataset} & $\epsilon$ & $\tau$ & $T$ \\
\cmidrule{1-4}
MNIST & 0.4 & 0.2 & 5 \\
SVHN & $4 / 255$ & $6 / 255$ & 10 \\
CIFAR-10 & $8 / 255$ & $14 / 255$ & 10 \\
\bottomrule
\end{tabular}
\end{table}

\end{document}